\documentclass{article}

\usepackage[preprint]{neurips_2026}

\usepackage[utf8]{inputenc} 
\usepackage[T1]{fontenc}    
\usepackage{hyperref}       
\usepackage{url}            
\usepackage{booktabs}       
\usepackage{amsfonts}       
\usepackage{nicefrac}       
\usepackage{microtype}      
\usepackage{xcolor}         
\usepackage{amsmath}
\usepackage{amssymb}
\usepackage{mathtools}
\usepackage{amsthm}
\usepackage{subcaption}
\usepackage{comment}
\usepackage{microtype}
\usepackage{graphicx}
\usepackage{booktabs} 
\usepackage{natbib}
\usepackage{algorithm}
\usepackage{algpseudocode}
\usepackage{hyperref}
\theoremstyle{plain}
\newtheorem{theorem}{Theorem}[section]
\newtheorem{proposition}[theorem]{Proposition}

\theoremstyle{definition}

\theoremstyle{remark}

\usepackage[textsize=tiny]{todonotes}
\usepackage{float}

\usepackage{hyperref}
\usepackage{perpage} \MakePerPage{footnote}
\usepackage{enumitem}
\usepackage{float}

\setlist[itemize]{leftmargin=4mm}
\setlist[enumerate]{leftmargin=4mm}

\title{Online Localized Conformal Prediction}

\author{
  Yuheng Lai\\
  University of Wisconsin - Madison\\
  \texttt{yuheng.lai@wisc.edu}
    \And
  Garvesh Raskutti\\
  University of Wisconsin - Madison\\
  \texttt{raskutti@wisc.edu}  
}

\begin{document}

\maketitle

\begin{abstract}
Conformal prediction is a framework that provides valid uncertainty quantification for general models with exchangeable data. However, in the online learning and time-series settings, exchangeability is not satisfied. Existing online conformal methods, such as adaptive conformal inference (ACI), can achieve long-run validity, yet they remain inefficient under covariate heterogeneity because they rely on global calibration. We propose \emph{Online Localized Conformal Prediction (OLCP)}, which combines online adaptation with covariate-dependent localization to better reflect heterogeneity. To reduce sensitivity to the localization bandwidth, we further develop \emph{OLCP-Hedge}, which performs bandwidth selection as an online expert aggregation problem using a constrained online convex optimization framework. Importantly, we provide coverage guarantees for both algorithms and demonstrate through simulations and real-data experiments that the proposed methods attain valid long-run coverage with narrower prediction sets than existing baselines.
\end{abstract}
\vspace{-10pt}

\section{Introduction}
\vspace{-7pt}
Reliable uncertainty quantification is essential in online learning and time-series prediction, where data are often temporally dependent, nonstationary, and heterogeneous across covariate space. Our goal is to construct online prediction set \(C_t(X_t)\subseteq\mathcal{Y}\) for a sequential data stream \(\{(X_t,Y_t)\}_{t=1}^T\) such that the long-run coverage target
\vspace{-3pt}
\begin{equation}
\label{eq:coverage_guarantee}
\frac{1}{T}\sum_{t=1}^T \mathbf{1}\{Y_t \in C_t(X_t)\} \approx 1-\alpha
\end{equation}
\vspace{-2pt}
is achieved while the average set size remains as small as possible.

Conformal prediction (CP) provides distribution-free uncertainty quantification with finite-sample marginal coverage under exchangeability \citep{vovk2005algorithmic, lei2018distribution, angelopoulos2023conformal}. However, exchangeability is often violated in online and time-series settings due to temporal dependence, distribution shift, heteroskedasticity, and structural breaks. Existing online conformal methods such as adaptive conformal inference (ACI) \citep{gibbs2021adaptive} and dynamically tuned ACI (DtACI) \citep{gibbs2024conformal} address this issue by adapting the nominal miscoverage level over time, and can achieve long-run validity under non-exchangeability. That being said, these methods remain \emph{globally} calibrated: they react to temporal changes in overall uncertainty, but do not adapt set size to local covariate-dependent heterogeneity. As a result, they can be simultaneously too wide in easy regions and too narrow in difficult regions of the covariate space.

A complementary line of work addresses heterogeneity through localization. In particular, localized conformal prediction (LCP) re-weights calibration points according to covariate similarity, allowing set size to vary with local uncertainty \citep{guan2023localized}. While attractive in heterogeneous regression problems, existing localized conformal methods rely on exchangeability and are not designed for online nonstationary data streams. This leaves a natural gap: current online conformal methods handle temporal non-exchangeability but remain global, whereas current localized conformal methods adapt to heterogeneity but do not handle non-exchangeable online data.

We propose \emph{Online Localized Conformal Prediction} (OLCP), which combines online calibration with covariate-dependent localization. At each time \(t\), OLCP computes a localized conformal quantile around the current covariate \(X_t\), and then updates the nominal miscoverage level using realized coverage feedback. Thus the set size adapts both over time and across the covariate space. A key practical issue is bandwidth selection. To reduce sensitivity to this choice, we formulate bandwidth selection as an online expert aggregation problem with prediction set size as the objective and coverage as the constraint, and develop \emph{OLCP-Hedge}, a constrained online convex optimization procedure over a collection of OLCP experts.

Our main contributions are as follows:
\vspace{-5pt}
\begin{itemize}
    \item We introduce OLCP, a localized online conformal method that combines covariate-dependent calibration with feedback-driven online adaptation.
    \item We show that OLCP enjoys a long-run coverage guarantee under its sequential coverage-tracking update, despite replacing global conformal quantiles with localized ones.
    \item We formulate OLCP bandwidth selection as a constrained online convex optimization problem and propose OLCP-Hedge, which controls long-run coverage violation while competing with the best feasible bandwidth expert in terms of set size. Both coverage and set size guarantees for OLCP-Hedge are provided.
    \item Through simulations and real-data experiments, we show that OLCP and OLCP-Hedge achieve valid long-run coverage with narrower prediction sets than existing online conformal baselines in heterogeneous and nonstationary settings.
\end{itemize}

\vspace{-8pt}
\section{Problem setup and related work}
\label{sec:related_work}
\vspace{-7pt}
We observe a sequential stream \(\{(X_t,Y_t)\}_{t=1}^T\), where \(X_t \in \mathcal X\) is observed before prediction and \(Y_t \in \mathcal Y\) is revealed afterward. At each time \(t\), using the past data together with the current covariate \(X_t\), we construct a prediction set
\[
C_t(X_t) \subseteq \mathcal Y
\]
for \(Y_t\). Our objective is to achieve the target long-run coverage level \eqref{eq:coverage_guarantee}
while keeping the prediction sets as small as possible. In contrast to the classical conformal setting, we do not assume exchangeability of \(\{(X_t,Y_t)\}\); the sequence may be temporally dependent, nonstationary, or heterogeneous across the covariate space.

\vspace{-3pt}
\paragraph{Conformal prediction (CP).}
We first introduce the split conformal prediction framework \citep{vovk2005algorithmic, angelopoulos2023conformal}. One uses an \emph{independent} training sample to construct a score function
\(
s:\mathcal X \times \mathcal Y \to \mathbb R,
\)
where larger values indicate that a candidate pair \((x,y)\) is less conforming to the training data; for example, one may take \(s(x,y)=|y-\hat f(x)|\), where \(\hat f\) is a pretrained predictor. 

Given calibration points \(Z_i=(X_i,Y_i)\), \(i=1,\dots,n\), define the calibration scores \(S_i := s(X_i,Y_i)\) and the empirical score distribution

\vspace{-15pt}
\[
\widehat F_n := \frac{1}{n}\sum_{i=1}^n \delta_{S_i}, \qquad \text{where } \delta_s \text{ is a point mass at } s \in \mathbb{R}.
\]
\vspace{-5pt}
For any \(\tau \in [0,1]\), let
\[
Q(\tau;F) := \inf\{t \in \mathbb R \cup \{\infty\}: F((-\infty,t]) \ge \tau\}
\]
denote the lower \(\tau\)-quantile of a distribution \(F\). Split conformal prediction calibrates the score threshold by
\[ \widehat q_{1-\alpha} := Q\!\left(\frac{\lceil (n+1)(1-\alpha)\rceil}{n};\,\widehat F_n\right), \]

\vspace{-10pt}
and outputs the prediction set
\[
C^{\text{CP}}(X_{n+1}) = \{y \in \mathcal Y : s(X_{n+1},y) \le \widehat q_{1-\alpha}\}.
\]
Under exchangeability of the calibration points and the test point, this construction yields finite-sample marginal coverage,
\[
\mathbb P\{Y_{n+1} \in C^{\text{CP}}(X_{n+1})\} \ge 1-\alpha,
\]
while allowing complete flexibility in the choice of score \(s\) \citep{lei2018distribution, romano2019conformalized}.

\paragraph{Localized conformal prediction (LCP).}
A natural extension of split conformal prediction is to replace the unweighted empirical score distribution with a weighted one, assigning larger weight to calibration points that are more relevant to the test point. This idea underlies weighted conformal methods for covariate shift and related distribution shifts \citep{hore2025conformal, barber2023conformal}. Localized conformal prediction (LCP) specializes this weighting scheme to the covariate space, emphasizing calibration points whose features are close to the test covariate \(X_{n+1}\) \citep{guan2023localized}. Specifically, LCP introduces a localizer
\(
H:\mathcal X \times \mathcal X \to [0,\infty),
\)
and defines normalized local weights
\[
p_{n+1,j}^H := \frac{H(X_{n+1},X_j)}{\sum_{k=1}^{n+1} H(X_{n+1},X_k)}, \qquad j=1,\dots,n+1,
\]
together with the weighted augmented score distribution
\vspace{-3pt}
\[
\widehat F_{n+1}^H
:=
\sum_{j=1}^{n} p_{n+1,j}^H \,\delta_{S_j}
\;+\;
p_{n+1,n+1}^H \,\delta_{\infty}.
\]
\vspace{-13pt}

A naive localized analogue of split conformal prediction would use the weighted quantile \(Q(1-\alpha;\widehat F_{n+1}^H)\), but this generally fails to preserve finite-sample validity. \cite{guan2023localized} shows that validity can be restored by replacing \(1-\alpha\) with a recalibrated level \(\tilde\alpha\), leading to the prediction set
\vspace{-1pt}
\[
C^{\text{LCP}}(X_{n+1})
=
\{y \in \mathcal Y : s(X_{n+1},y) \le Q(\tilde\alpha;\widehat F_{n+1}^H)\}.
\]
This retains finite-sample marginal coverage while allowing the prediction set to adapt to covariate-dependent heterogeneity.

\vspace{-3pt}
\paragraph{Adaptive conformal inference (ACI).}
A different line of work addresses non-exchangeability through online calibration. Following \cite{gibbs2021adaptive}, let \(D_t\) denote an estimated conformity-score distribution at time \(t\), for example the empirical distribution of recent scores. For any \(\beta \in [0,1]\), define
\[
C^{\text{ACI}}_t(\beta)
=
\{y \in \mathcal Y : s(X_t, y) \le Q(1-\beta; D_t)\}.
\]
Adaptive conformal inference (ACI) replaces the fixed nominal level \(\alpha\) by an online-updated parameter \(\alpha^{\text{ACI}}_t\). After predicting with \(C^{\text{ACI}}_t(\alpha^{\text{ACI}}_t)\), it observes
\vspace{-1pt}
\[
\mathrm{err}^{\text{ACI}}_t := \mathbf{1}\{Y_t \notin C^{\text{ACI}}_t(\alpha^{\text{ACI}}_t)\},
\]
and updates
\[
\alpha^{\text{ACI}}_{t+1}
=
\alpha^{\text{ACI}}_t + \gamma(\alpha-\mathrm{err}^{\text{ACI}}_t),
\]
where \(\gamma>0\) is a step size \citep{gibbs2021adaptive}. This yields a feedback-driven calibration rule that targets long-run coverage under arbitrary distribution shifts.

\vspace{-3pt}
\paragraph{Other related work.}
A key limitation of ACI is its sensitivity to the step size \(\gamma\). DtACI \citep{gibbs2024conformal} and AgACI \citep{zaffran2022adaptive} address this by aggregating multiple ACI experts online, while recent work develops parameter-free online conformal updates based on universal portfolio algorithms \citep{liu2026online}. Other sequential conformal methods update scores or interval sizes over time: EnbPI \citep{enbpi} builds prediction sets around bootstrap ensemble predictors under weak dependence conditions, and SPCI \citep{spci} forecasts future residual quantiles from past residuals. A different direction anticipates abrupt shifts using additional structure: CPTC \citep{sun2025cptc} maintains state-specific score sets and aggregates them using a predicted latent state sequence, but relies on a state predictor and state-conditioned forecaster. Under exchangeability, selection and aggregation methods choose among conformal predictors while preserving finite-sample validity through recalibration \citep{yang2025selection, liang2024conformal}; our setting instead treats coverage as a long-run online constraint. Finally, conformal methods beyond exchangeability use data-dependent or fixed weights, including covariate-shift conformal prediction \citep{tibshirani2019conformal} and fixed-weight approaches for non-exchangeable data \citep{barber2023conformal}.

\vspace{-7pt}
\section{Methods}
\vspace{-3pt}
\subsection{Online Localized Conformal Prediction}
\label{sec:olcp}

We now introduce our first and main method \emph{Online Localized Conformal Prediction} (OLCP).
The starting point is that two distinct difficulties arise in sequential prediction. First,
under covariate heterogeneity, uncertainty can vary substantially across the feature space, so
a global conformal quantile may allocate set size inefficiently. Second, under temporal
distribution shift, the appropriate calibration level itself changes over time. OLCP addresses
both issues simultaneously by combining localized calibration in \(X\) with online updates of
the nominal level.

\paragraph{Localized calibration distribution.}
Fix a bandwidth \(h>0\), a localizer
\vspace{-3pt}
\[
H_h:\mathcal X\times\mathcal X\to[0,\infty),
\]

\vspace{-8pt}
and a rolling calibration window 
\(
\mathcal I_t := \{\max(1, t-R),\dots,t-1\}
\) with fixed window size $R$. For each \(i\in \mathcal I_t\), \(S_i = s(X_i, Y_i)\) denotes the realized conformity score at time \(i\). For a query covariate \(x\in\mathcal X\), define normalized local weights
\vspace{-5pt}
\begin{equation}
w_{t,i}^{(h)}(x)
:=
\frac{H_h(x,X_i)}
{\sum_{j\in\mathcal I_t} H_h(x,X_j)}.
\label{eq:weights}
\end{equation}

\vspace{-10pt}
We then form the localized empirical distribution
\vspace{-4pt}
\[
D_t^{(h)}(x)
:=
\sum_{i\in\mathcal I_t} w_{t,i}^{(h)}(x)\,\delta_{S_i}.
\]

\vspace{-10pt}
This is the online analogue of the weighted empirical distributions used in localized
conformal prediction \citep{guan2023localized}: past scores associated with covariates close
to the query point \(x\) receive more weight.

For any query covariate \(x\in\mathcal X\) and any \(\beta\in[0,1]\), define
\vspace{-5pt}
\[
C_t^{(h)}(x;\beta)
:=
\{y\in\mathcal Y:\; s(x,y)\le Q(1-\beta;D_t^{(h)}(x))\}.
\]

\vspace{-7pt}
Thus \(h\) controls the
degree of localization: small \(h\) emphasizes nearby covariates, while large \(h\) recovers a
more global rule \citep{guan2023localized}.

\vspace{-5pt}
\paragraph{OLCP update.}
OLCP maintains a nominal level \(\alpha_t\in[0,1]\). At time \(t\), it outputs $C_t^{(h)}(X_t; \alpha_t)$,
observes the error
\[
\mathrm{err}_t
:=
\mathbf 1\{Y_t\notin C_t^{(h)}(X_t;\alpha_t)\},
\]
and updates
\begin{equation}
\alpha_{t+1}
=
\Pi_{[0,1]}\bigl(\alpha_t+\gamma(\alpha-\mathrm{err}_t)\bigr),
\label{eq:olcp-update}
\end{equation}
where \(\Pi_{[0,1]}\) is projection onto \([0,1]\) and \(\gamma>0\) is a step size.

The algebraic form of \eqref{eq:olcp-update} is similar to ACI, but the object being calibrated is
different: OLCP updates the level for the localized family \(\{C_t^{(h)}(X_t;\beta):\beta\in[0,1]\}\),
not for a single global conformal rule. This distinction is what allows OLCP to adapt set size across both time and covariate space. 

\begin{algorithm}[t]
\caption{Online Localized Conformal Prediction (OLCP)}
\label{alg:olcp}
\begin{algorithmic}[1]
\Require target miscoverage \(\alpha\), step size \(\gamma\), bandwidth \(h\), window length \(R\)
\State Initialize \(\alpha_1\in[0,1]\)
\For{\(t=1,2,\dots,T\)}
    \State Compute localized weights \(w_{t,i}^{(h)}(X_t)\) for \(i\in\mathcal I_t=\{\max(1, t-R),\dots,t-1\}\) from \eqref{eq:weights}
    \State Form \(D_t^{(h)}(X_t)=\sum_{i\in\mathcal I_t} w_{t,i}^{(h)}(X_t)\delta_{S_i}\)
    \State Output \(C_t^{(h)}(X_t;\alpha_t)=\{y:s(X_t, y)\le Q(1-\alpha_t;D_t^{(h)}(X_t))\}\)
    \State Observe \(Y_t\) and set \(\mathrm{err}_t=\mathbf 1\{Y_t\notin C_t^{(h)}(X_t;\alpha_t)\}\)
    \State Update \(\alpha_{t+1}=\Pi_{[0,1]}(\alpha_t+\gamma(\alpha-\mathrm{err}_t))\)
\EndFor
\end{algorithmic}
\end{algorithm}

\vspace{-3pt}
\paragraph{Pinball-loss view and coverage guarantee.}
The update in \eqref{eq:olcp-update} also admits a useful optimization interpretation. Define the coverage-boundary level
\[
\beta_t^{(h)}
:=
\sup\{\beta\in[0,1]: Y_t\in C_t^{(h)}(X_t;\beta)\},
\]
namely, the boundary nominal level for which the realized response starts to be uncovered. Because \(C_t^{(h)}(X_t;\beta)\) is decreasing in \(\beta\),
\vspace{-4pt}
\[
\alpha_t<\beta_t^{(h)} \Rightarrow \mathrm{err}_t=0,
\qquad
\alpha_t>\beta_t^{(h)} \Rightarrow \mathrm{err}_t=1.
\]
\vspace{-3pt}
At \(\alpha_t=\beta_t^{(h)}\), either outcome may occur. For pinball loss
\[
\ell(\beta,\theta)=\alpha(\beta-\theta)-\min\{0,\beta-\theta\},
\]
we therefore have
\[
\mathrm{err}_t-\alpha
\in
\partial_\theta \ell(\beta_t^{(h)},\theta)\big|_{\theta=\alpha_t},
\]
so \eqref{eq:olcp-update} is exactly projected online gradient descent on the loss sequence
\(\{\ell(\beta_t^{(h)},\cdot)\}_{t=1}^T\). 

We introduce the following long-term coverage guarantee for OLCP with proof deferred in Appendix~\ref{app:proof-projected-olcp}.
\begin{proposition}[Boundary-corrected coverage control for OLCP]
\label{prop:projected-olcp-boundary}
Define
$
z_t:=\alpha_t+\gamma(\alpha-\mathrm{err}_t),
$
and the lower- and upper-boundary projection corrections
$
L_t:=(-z_t)_+,
U_t:=(z_t-1)_+.
$
Equivalently,
$
\alpha_{t+1}=z_t+L_t-U_t.
$
Then, for every \(T\ge 1\) and fixed $\gamma > 0$,
\vspace{-7pt}
\[
\sum_{t=1}^T(\mathrm{err}_t-\alpha)
=
\frac{\alpha_1-\alpha_{T+1}}{\gamma}
+
\frac{1}{\gamma}\sum_{t=1}^T(L_t-U_t).
\]\vspace{-10pt}

In particular, if 
$
\sum_{t=1}^T L_t=o(T),
$
then
\(
\limsup_{T\to\infty}
\
\dfrac1T\sum_{t=1}^T\mathrm{err}_t
\le \alpha.
\)\vspace{-5pt}

If additionally
$
\sum_{t=1}^T (L_t+U_t)=o(T),
$
then
\(
\dfrac1T\sum_{t=1}^T\mathrm{err}_t\to \alpha.
\)
\end{proposition}

\vspace{-7pt}
\paragraph{Remark on boundary terms.}
Proposition~\ref{prop:projected-olcp-boundary} shows exactly what projection changes relative to the unprojected ACI. Unlike ACI, which can use infinite or empty prediction sets to obtain exact telescoping, OLCP keeps prediction sets finite and must account for \(L_t\) and \(U_t\). The lower correction \(L_t\) measures unresolved undercoverage pressure: it is positive only when the update wants to move \(\alpha_t\) below \(0\), i.e. when OLCP misses while already using a very wide finite set. The upper correction \(U_t\) is the analogous overcoverage pressure near \(\alpha_t=1\). Thus projected OLCP recovers ACI-style long-run calibration whenever these normalized boundary corrections are negligible.  Appendix~\ref{app:olcp-boundary} gives sufficient conditions under which these terms vanish and explains how to diagnose them empirically.

\vspace{-4pt}
\subsection{OLCP-Hedge: constrained aggregation across localization bandwidths}
\label{sec:sv-olcp-adahedge}

\vspace{-4pt}
OLCP requires choosing a localization bandwidth \(h\). This bandwidth controls a familiar
bias--variance trade-off: small \(h\) adapts strongly to local heterogeneity but can be noisy,
whereas large \(h\) is more stable but may wash out covariate-dependent uncertainty
\citep{guan2023localized}. Since the best localization scale is generally unknown and may vary
over time, this motivates aggregating a collection of OLCP experts rather than committing to a single bandwidth in advance.

This is related to the classical prediction-with-expert-advice problem, where the learner
combines expert predictions and competes with the best expert in hindsight
\citep{orabona2019modern}. However, our objective is
not a single scalar loss. If set size is used as the loss, ordinary Hedge may favor narrow
experts that undercover; if miscoverage is used as the loss, the resulting prediction sets may be valid
but inefficient. Thus bandwidth selection is naturally a constrained online problem: minimize set size
while controlling long-run miscoverage.

\vspace{-4pt}
\paragraph{OLCP expert pool.}
Fix \(K\) OLCP procedures, for example, with different bandwidths \(\{h_i:i = 1,\ldots, K\}\). At time \(t\), expert \(i\) outputs a
prediction set \(C_{t,i}(X_t)\), for example $C_{t,i}(X_t):=C_t^{(h_i)}(X_t;\alpha_t)$\footnote{Our algorithm can also be applied to aggregate OLCP procedures with different step size $\gamma$, see more on Section~\ref{sec:discussion}.}, and receives feedback of 
$$
\text{size } \omega_{t,i}:= \text{\texttt{size}}(C_{t,i}(X_t)) \qquad \text{and} \qquad
\text{miscoverage }\mathrm{err}_{t,i}:=\mathbf 1\{Y_t\notin C_{t,i}(X_t)\},
$$
where \texttt{size}($\cdot$) could be interval width, or cardinality depending on the tasks.

Let \(\omega_t=(\omega_{t,1},\dots,\omega_{t,K})\) and
\(e_t=(\mathrm{err}_{t,1},\dots,\mathrm{err}_{t,K})\). We maintain a distribution \(p_t\in\Delta_K\) over experts,
sample \(I_t\sim p_t\), and output
\(
C_{t,I_t}(X_t)
\),

\vspace{-5pt}
\paragraph{Size objective and miscoverage constraint.}
For \(p\in\Delta_K\), define
\[
f_t(p):=\langle \omega_t,p\rangle,
\qquad
g_t(p):=\langle e_t,p\rangle-\alpha .
\]
Then
\[
\mathbb E_{I_t\sim p_t}[\omega_{t, I_t}]=f_t(p_t),
\qquad
\mathbb E_{I_t\sim p_t}[\mathrm{err}_{t, I_t}]=\langle e_t,p_t\rangle=g_t(p_t)+\alpha.
\]
Thus \(g_t(p_t)\) is the expected excess miscoverage of the aggregate. Equivalently, the
meta-objective is
\vspace{-3pt}
\[
\min_{p_t\in\Delta_K}\sum_{t=1}^T f_t(p_t)
\qquad
\text{while keeping}
\qquad
\sum_{t=1}^T (g_t(p_t))_+
\ \text{sublinear}.
\]
This is an instance of \emph{constrained online convex optimization} (COCO): the learner
chooses an action before observing a convex loss and constraint, and aims to control both
regret and cumulative constraint violation. COCO has a substantial recent literature on
time-varying constraints and long-term feasibility; see, e.g.,
\cite{guo2022online,sinha2024optimal} for recent overviews and comparisons. 

We adapt the algorithm proposed in  \cite{sinha2024optimal} because it gives state-of-the-art simultaneous guarantees,
\(O(\sqrt T)\) regret and \(\widetilde O(\sqrt T)\) cumulative constraint violation, without
Slater-type assumptions or per-round constrained optimization. Their reduction combines a
Lyapunov-style virtual queue with a black-box adaptive OCO subroutine. 

The original paper focuses on the Euclidean space; here we extend their algorithm to the probability simplex of OLCP experts, and prove the corresponding size-regret and excess-miscoverage bounds in the
probability simplex geometry.

\vspace{-3pt}
\paragraph{Assumptions.}
We use the following assumptions, specialized to our expert aggregation problem.
\vspace{-3pt}
\begin{itemize}
    \item \textbf{Assumption A (bounded sizes).}
There exists \(G>0\) such that for all \(t\),
\[
\|\omega_t\|_\infty\le G,
\qquad
\|e_t\|_\infty\le G .
\]
Since \(e_{t,i}\in\{0,1\}\), this only requires a uniform bound on set sizes. This assumption is mild after size normalization or when the response range is bounded.

\item 
\textbf{Assumption B (uniformly feasible comparator).}
There exists \(u^\star\in\Delta_K\) such that
\[
g_t(u^\star)=\langle e_t,u^\star\rangle-\alpha\le 0,
\qquad
\forall t.
\]
This assumption is
a standard but strong feasibility condition in COCO: one compares against a fixed feasible action. In our
setting, it can be enforced by including a conservative expert, although this may increase the
comparator size. If exact feasibility is unavailable, the same viewpoint suggests a relaxed fallback using online
constraint-satisfaction ideas from \cite{sinha2024optimal}: instead of requiring a uniformly feasible
comparator, one can compare to an \(S\)-feasible or \(P_T\)-constrained benchmark and aim for
sublinear excess miscoverage relative to that weaker benchmark. We discuss diagnostics for this assumption in Appendix~\ref{app:fea_dia}.

\end{itemize}

\vspace{-5pt}
\paragraph{Surrogate loss.}
Following \cite{sinha2024optimal}, OLCP-Hedge maintains a virtual queue for excess
miscoverage and feeds AdaHedge\footnote{AdaHedge is only one possible subroutine; see the modularity discussion below in \ref{para:adahedge_remark}.}\citep{derooij2014follow} a surrogate loss that combines set size with queue-weighted
constraint violation. Let \(\mathcal Q(0)=0\) and update
\[
\mathcal Q(t)=\mathcal Q(t-1)+\kappa(g_t(p_t))_+, \qquad\text{where }\kappa>0 \text{ is a constant}.
\]
With \(\Phi(q)=e^{\lambda q}-1\) and $\lambda >0$, define
\[
\hat f_t(p)
=
V\kappa f_t(p)
+
\Phi'(\mathcal Q(t))\,\kappa(g_t(p))_+, \qquad\text{where }V>0 \text{ is a constant}.
\]
The first term penalizes prediction set size, while the second term increasingly penalizes
excess miscoverage as the queue grows. The full algorithm is detailed in Algorithm~\ref{alg:sv-olcp-adahedge} with more details of AdaHedge in Appendix~\ref{app:adahedge}.
\begin{algorithm}[t]
\caption{OLCP-Hedge}
\label{alg:sv-olcp-adahedge}
\begin{algorithmic}[1]
\Require target miscoverage \(\alpha\), OLCP experts \(\{\{C_{t,i}\}_{t=1}^T\}_{i=1}^K\), parameters \(V,\kappa,\lambda\)
\State Initialize \(\mathcal{Q}(0)=0\) and AdaHedge on \(\Delta_K\)
\For{\(t=1,2,\dots,T\)}
    \State AdaHedge outputs \(p_t\in\Delta_K\)
    \State Sample \(I_t\sim p_t\) and output \(C_{t,I_t}(X_t)\)
    \State Observe \(Y_t\); compute the size \(\omega_{t,i}\) and \(\mathrm{err}_{t,i}=\mathbf 1\{Y_t\notin C_{t,i}(X_t)\}\) for all \(i\)
    \State Define \(f_t(p)=\langle \omega_t,p\rangle\), \(g_t(p)=\langle e_t,p\rangle-\alpha\)
    \State Update \(\mathcal{Q}(t)=\mathcal{Q}(t-1)+\kappa(g_t(p_t))_+\)
    \State Form \(\hat f_t(p)=V\kappa f_t(p)+\Phi'(\mathcal{Q}(t))\kappa(g_t(p))_+\)
    \State Choose \(\xi_t\in\partial \hat f_t(p_t)\) and feed the linearized loss
    \(\ell_t(p)=\langle \xi_t,p\rangle\) to AdaHedge
\EndFor
\end{algorithmic}
\end{algorithm}

\vspace{-5pt}
\paragraph{Guarantee.}
The next theorem states that OLCP-Hedge competes with the best feasible expert mixture in
expected size while keeping cumulative expected excess miscoverage sublinear. The proof is
deferred to Appendix~\ref{app:sv-proof}.

\begin{theorem}[Size regret and excess-miscoverage control]
\label{thm:sv-olcp}
Assume Assumptions~A--B. Run Algorithm~\ref{alg:sv-olcp-adahedge} with
\vspace{-5pt}
\[
  V=1,\qquad C_{\mathrm{AH}}:=2 \sqrt{4+\ln K}, \qquad
  \kappa := \frac{1}{\sqrt{2}\,C_{\mathrm{AH}}\,G},\qquad
  \lambda := \frac{1}{2\sqrt{T}}.
\]

\vspace{-10pt}
Then, for any feasible comparator \(u^\star\in\Delta_K\), we have the following expected set-size regret bound:
\vspace{-5pt}
\[
  \sum_{t=1}^T
  \Bigl(
    \mathbb E_{I_t\sim p_t}[\omega_{t, I_t}]
    -
    \langle \omega_t,u^\star\rangle
  \Bigr)
  \le
  4G\sqrt{2(4+\ln K)T}
  =
  O\!\left(G\sqrt{T(1+\ln K)}\right),
\]

\vspace{-10pt}
and the following cumulative excess-miscoverage bound:
\vspace{-5pt}
\[
  \sum_{t=1}^T
  \Bigl(
    \mathbb E_{I_t\sim p_t}[\mathrm{err}_{t, I_t}]-\alpha
  \Bigr)_+
  \le
  4G\sqrt{2(4+\ln K)T}\;
  \ln\!\Bigl(2+\bigl(2+\tfrac{\sqrt2}{2}\bigr)T\Bigr)
  =
  \widetilde O\!\left(G\sqrt{T(1+\ln K)}\right).
\]
\vspace{1pt}
\end{theorem}
\vspace{-10pt}
\paragraph{Remark on the modularity of the subroutine.}
\label{para:adahedge_remark}
The OLCP-Hedge analysis is modular in the OCO subroutine, following the black-box philosophy of \cite{sinha2024optimal}. The proof only uses that, when the subroutine is fed the linearized surrogate losses
\(\ell_t(p)=\langle \xi_t,p\rangle\), it satisfies a data-dependent regret bound of the form
\vspace{-6pt}
\[
  \sum_{t=1}^T \langle \xi_t,p_t-u\rangle
  \le
  C\sqrt{\sum_{t=1}^T\|\xi_t\|_\infty^2},
  \qquad \forall u\in\Delta_K.
\]
AdaHedge is one parameter-free expert algorithm with such a worst-case/adaptive guarantee \citep{derooij2014follow,orabona2019modern}. More broadly, related adaptive expert algorithms providing comparable data-dependent regret guarantees could be used in the same reduction after replacing \(C_{\mathrm{AH}}\) by the corresponding constant or regret complexity \citep{gaillard2014second,koolen2015second,orabona2016coin}. Thus the theorem is not tied to AdaHedge itself.

\vspace{-6pt}
\section{Experiments}
\label{sec:exp}
\vspace{-3pt}
We evaluate OLCP and OLCP-Hedge on both synthetic experiments and real sequential prediction tasks. We first introduce our experimental setup, then present simulation and real-data results. Code is available  \href{https://github.com/yuhenglai/OLCP}{online}.
\vspace{-5pt}
\subsection{Experimental setup}
\label{sec:exp-setup}
\vspace{-3pt}
We compare seven methods (see Section~\ref{sec:related_work}): CP \citep{ lei2018distribution}, LCP \citep{guan2023localized}, ACI \citep{gibbs2021adaptive}, DtACI \citep{gibbs2024conformal}, SPCI \citep{spci}, and proposed OLCP/OLCP-Hedge. All methods use the same base predictor, conformity scores, and rolling calibration window. For localized methods (LCP/OLCP/OLCP-Hedge), we use an exponential kernel
\vspace{-5pt}
\[
H_h(x,x')
=
\exp\!\left(-\frac{\|\tilde x-\tilde x'\|_2}{h}\right),
\]
where covariates are standardized within the current calibration window before computing distances. The base bandwidth \(h_0\) is chosen by a Silverman-style rule, and OLCP-Hedge aggregates the grid
\vspace{-5pt}
\[
h\in\{0.5,0.75,1,1.25,1.5\}h_0 .
\]
\vspace{-20pt}

All calibration is online: at time \(t\), each method uses only past conformity scores in the rolling window. Full implementation details, including bandwidth formulas, adaptive step sizes, and SPCI settings, are given in Appendix~\ref{app:method}.

All methods are evaluated using empirical coverage and average set size, for prediction set (interval) $C_t(X_t) = [L_t(X_t), U_t(X_t)]$ produced by any methods above, we measure
\vspace{-4pt}
\[
\widehat{\mathrm{cov}}
=
\frac1T\sum_{t=1}^T \mathbf 1\{Y_t\in C_t(X_t)\}, \qquad
\widehat{\mathrm{size}}
=
\frac1T\sum_{t=1}^T
\bigl(U_t(X_t)-L_t(X_t)\bigr). 
\]

\vspace{-15pt}
\subsection{Simulations}
\label{sec:sim}

\vspace{-5pt}
We first use controlled synthetic experiments to isolate two failure modes of existing online conformal methods: covariate-dependent heterogeneity and abrupt temporal distribution shift. We simulate a univariate time series \(\{Y_t\}_{t=0}^T\) and form a one-step-ahead prediction task with
\(
X_t=Y_{t-1}.\)
A fixed linear predictor \(\hat f\) is trained on the first \(500\) observations, and all methods use the absolute residual score
\(
S_t=|Y_t-\hat f(X_t)|.
\)
We use \(T=1{,}500\), rolling calibration window \(R=200\), target miscoverage \(\alpha=0.1\), and \(100\) Monte Carlo repetitions. All the results are reported on the test set.

Let \(\varepsilon_t\sim N(0,1)\). We consider three scenarios:
\vspace{-3pt}
\begin{itemize}
    \item \textbf{A: Stationary:}
    \(Y_t=0.5Y_{t-1}+\varepsilon_t\).
    \item \textbf{B: Heterogeneous:}
    \(Y_t=0.5Y_{t-1}+\sigma_t\varepsilon_t\), where
    \(\sigma_t=\min\{\exp(0.25Y_{t-1}),10\}\).
    \item \textbf{C: Change point:}
    \(Y_t=\phi_tY_{t-1}+\varepsilon_t\), where
    \(
    \phi_t=
    \begin{cases}
    0.8, & t\le T/2,\\
    -0.8, & t>T/2.
    \end{cases}
    \)
\end{itemize}

\vspace{-10pt}
Scenario A is a sanity check, Scenario B tests adaptation to covariate-dependent noise, and Scenario C tests adaptation to abrupt temporal shift.

\begin{table}[t]
\vspace{-15pt}
\centering
\caption{Simulation results over \(100\) repetitions. Each entry reports mean (standard deviation). 
Boldface marks the smallest average size among methods whose coverage attains the \(0.90\) target.}
\label{tab:sim-summary}
\small
\setlength{\tabcolsep}{3.8pt}
\begin{tabular}{lcccccc}
\toprule
 & \multicolumn{2}{c}{A: Stationary} & \multicolumn{2}{c}{B: Heterogeneous} & \multicolumn{2}{c}{C: Change point} \\
\cmidrule(lr){2-3}\cmidrule(lr){4-5}\cmidrule(lr){6-7}
Method & Coverage & Size & Coverage & Size & Coverage & Size \\
\midrule
CP
& \textbf{0.900 (0.005)} & \textbf{3.301 (0.086)}
& 0.900 (0.005) & 3.611 (0.142)
& 0.869 (0.009) & 7.363 (0.468) \\
LCP
& 0.892 (0.005) & 3.271 (0.085)
& 0.894 (0.006) & 3.448 (0.106)
& 0.879 (0.008) & 6.083 (0.289) \\
ACI
& 0.900 (0.003) & 3.341 (0.095)
& 0.900 (0.003) & 3.738 (0.192)
& 0.900 (0.004) & 7.908 (0.477) \\
DtACI
& 0.900 (0.003) & 3.326 (0.094)
& 0.900 (0.004) & 3.692 (0.171)
& 0.898 (0.004) & 7.840 (0.466) \\
SPCI
& 0.812 (0.010) & 2.850 (0.070)
& 0.819 (0.009) & 3.056 (0.109)
& 0.773 (0.014) & 3.371 (0.116) \\
\midrule
OLCP
& 0.900 (0.003) & 3.382 (0.100)
& 0.901 (0.003) & 3.572 (0.129)
& 0.901 (0.003) & 6.349 (0.278) \\
OLCP-Hedge
& 0.900 (0.003) & 3.376 (0.097)
& \textbf{0.900 (0.003)} & \textbf{3.564 (0.128)}
& \textbf{0.900 (0.004)} & \textbf{6.348 (0.278)} \\
\bottomrule
\end{tabular}
\end{table}

\begin{figure}[t]
  \centering
  \vspace{-10pt}\includegraphics[width=1\linewidth]{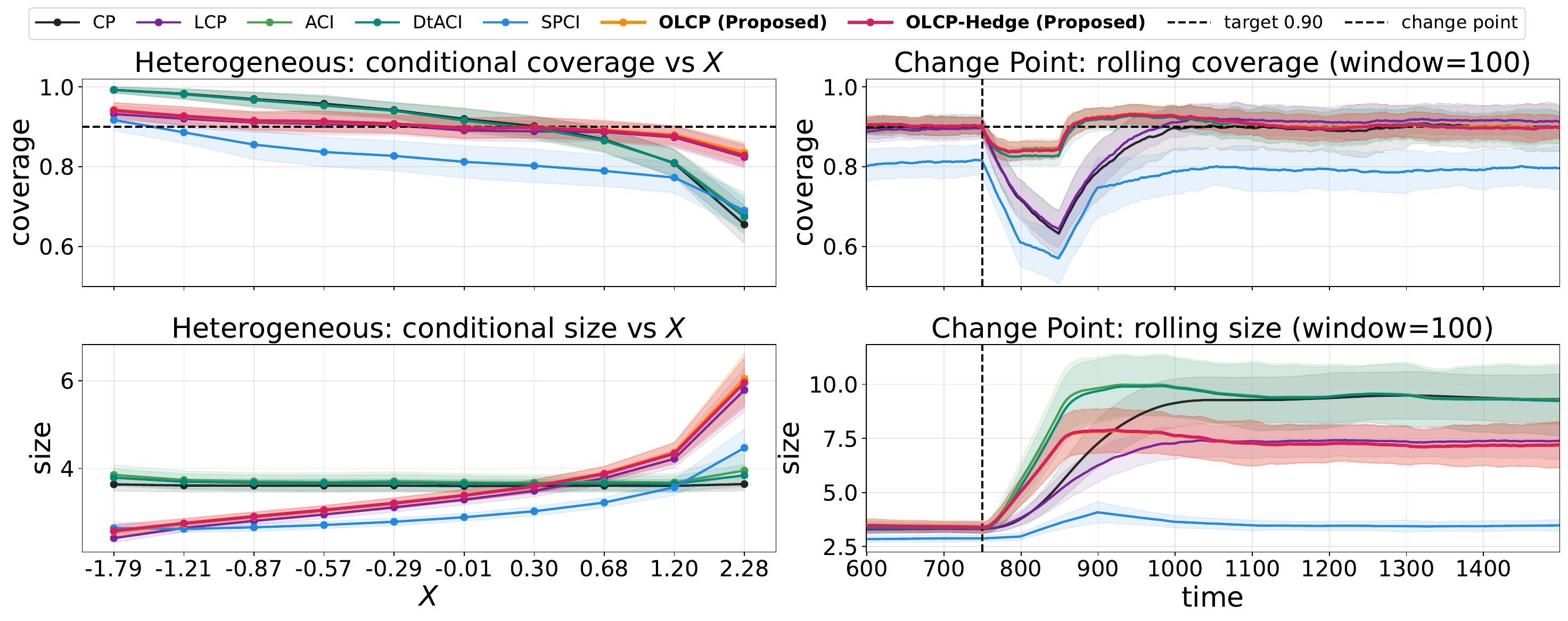}
  \vspace{-20pt}\caption{\textbf{Diagnostics for simulation.}
Left panels show Scenario B conditional coverage and average size across \(X_t\). Right panels show Scenario C rolling coverage and rolling average size with window size \(100\); the vertical dashed line marks the change point. Shaded bands show mean \(\pm\) one standard deviation across repetitions. OLCP (DtACI) curves are partially hidden behind OLCP-Hedge (ACI).}
  \label{fig:simBC_two_col}
  \vspace{-15pt}
\end{figure}

\vspace{-7pt}
\paragraph{Results.}
Table~\ref{tab:sim-summary} summarizes marginal coverage and average prediction set size, with running time of each method reported in Appendix~\ref{app:running_time}. Scenario A is a sanity check: in the stationary homoskedastic case, there is little structure for either localization or adaptation to exploit, and the valid methods have similar sizes. 

Scenario B isolates covariate-dependent heterogeneity. Global methods (CP, ACI, DtACI) achieve reasonable marginal coverage, but Figure~\ref{fig:simBC_two_col} shows that their prediction-set sizes are nearly constant in \(X_t\), causing undercoverage in high-noise regions. LCP localizes the residual quantile but mildly undercovers due to its fixed nominal level. OLCP and OLCP-Hedge combine localization with online calibration, expanding sets where noise is high while maintaining near-nominal marginal coverage; OLCP-Hedge gives the smallest size among near-valid methods.

Scenario C isolates abrupt temporal shift. CP and LCP recover slowly after the change point and undercover, while ACI and DtACI restore coverage mainly through global size inflation. OLCP and OLCP-Hedge recover to the target coverage level with substantially smaller sets, showing that localization improves efficiency under heterogeneous uncertainty and online calibration restores validity under temporal shift.

\vspace{-5pt}
\subsection{Real-data experiments}
\label{sec:realdata}
\vspace{-5pt}
We evaluate the proposed methods on three real time-series datasets; additional preprocessing details, model hyperparameters, data splits, and diagnostic plots are deferred to Appendix~\ref{app:real-data}.

\vspace{-4pt}
\begin{itemize}
    \item \textbf{ELEC2.}
    ELEC2 contains electricity market prices, demands, and transfers from New South Wales and Victoria \citep{harries1999splice}. We predict transfer using the two states' prices and demands as covariates. After removing the initial constant-response segment, the dataset has \(27{,}552\) rows; the base predictor is a gradient-boosted regression tree.

    \item \textbf{ILINet.}
    ILINet is a weekly CDC influenza-like illness surveillance dataset with \(1{,}305\) observations \citep{cdcFluView,darts}. We predict the state-population-weighted weekly patient percentage using a lag window of \(26\) past responses as covariates, with a temporal convolutional network (TCN) as the base predictor \citep{lea2016temporalconvolutionalnetworksunified}.

    \item \textbf{ETF volatility.}
    We forecast absolute daily log returns for five ETFs in recent $18$ years ($4{,}742$ trading days): SPY, QQQ, IWM, EEM, and TLT \citep{etf}. Covariates are a lag window of \(30\) past observations and the lagged VIX index \citep{cboeVIX}; the base predictor is a TCN. Widths are reported in percentage points of absolute log return.
\end{itemize}
\vspace{-5pt}

Table~\ref{tab:real-data} summarizes the real-data results with running time in Appendix~\ref{app:running_time}. Standard errors are computed using a block bootstrap over time, with block lengths \(48/26/20\) for ELEC2/ILINet/ETF volatility, respectively and \(1000\) bootstrap replicates. Across datasets, fixed LCP often yields shorter intervals but can undercover, while ACI and DtACI restore nominal coverage at the cost of wider intervals. OLCP and OLCP-Hedge provide the strongest overall coverage--efficiency tradeoff: OLCP-Hedge is consistently near-nominal and is the most efficient among the near-valid methods.

The rolling diagnostics in Appendix~\ref{app:real-data} further clarify this pattern. Figures~\ref{fig:elec2-roll}--\ref{fig:etf-roll} show that ACI and DtACI tend to maintain coverage by raising interval sizes globally, whereas OLCP and OLCP-Hedge track the target with lower rolling size over much of the test period. The ETF volatility regime analysis in Table~\ref{tab:vix-regime} gives a more targeted view: in low-volatility periods, all near-valid methods over-cover, but OLCP and OLCP-Hedge achieve the smallest sizes among them. In high-volatility periods, OLCP attains the best trade-off with coverage \(0.890\) with size \(4.004\), improving on both ACI (\(0.890\), \(4.349\)) and DtACI (\(0.889\), \(4.184\)). These diagnostics reinforce the central message: localization improves efficiency by adapting set size to the covariate regime, while online calibration maintains long-run coverage under temporal shift.
\begin{table}[t]
\centering
\caption{Real-data results. Each entry reports mean (block-bootstrap standard error). Coverage is empirical marginal coverage; Size is average interval width. For ETF volatility, Size is reported as percentage points of absolute daily log return. ELEC2 and ILINet sizes are in their normalized response scales.}
\label{tab:real-data}
\small
\setlength{\tabcolsep}{3.5pt}
\begin{tabular}{lcccccc}
\toprule
& \multicolumn{2}{c}{ELEC2}
& \multicolumn{2}{c}{ILINet}
& \multicolumn{2}{c}{ETF volatility} \\
\cmidrule(lr){2-3}
\cmidrule(lr){4-5}
\cmidrule(lr){6-7}
Method
& Coverage & Size
& Coverage & Size
& Coverage & Size (\%) \\
\midrule
CP
& 0.890 (0.007) & 0.3513 (0.008)
& 0.897 (0.034) & 0.9366 (0.134)
& 0.905 (0.005) & 2.6890 (0.062) \\
LCP
& 0.896 (0.006) & 0.3312 (0.007)
& 0.908 (0.030) & 0.7810 (0.093)
& 0.896 (0.005) & 2.4269 (0.053) \\
ACI
& 0.900 (0.006) & 0.3675 (0.009)
& 0.900 (0.033) & 0.9004 (0.142)
& 0.901 (0.004) & 2.6741 (0.088) \\
DtACI
& 0.898 (0.006) & 0.3660 (0.009)
& 0.900 (0.033) & 0.8886 (0.141)
& 0.900 (0.004) & 2.6178 (0.082) \\
SPCI
& 0.777 (0.008) & 0.1780 (0.002)
& 0.794 (0.043) & 0.5110 (0.075)
& 0.842 (0.006) & 2.2495 (0.056) \\
\midrule
OLCP
& \textbf{0.900 (0.005)} & \textbf{0.3377 (0.008)}
& 0.904 (0.027) & 0.7917 (0.110)
& 0.900 (0.004) & 2.5587 (0.083) \\
OLCP-Hedge
& 0.899 (0.005) & 0.3351 (0.008)
& \textbf{0.900 (0.027)} & \textbf{0.7478 (0.105)}
& \textbf{0.900 (0.004)} & \textbf{2.5516 (0.079)} \\
\bottomrule
\vspace{-23pt}
\end{tabular}
\end{table}

\vspace{-10pt}
\section{Conclusion and future work}
\vspace{-5pt}
\label{sec:discussion}
Our results show that localization improves efficiency under heterogeneity, while online update restores coverage under temporal shift. However, several limitations remain. Like other localized conformal methods \citep{guan2023localized}, OLCP also depends on informative covariates and a suitable distance metric. OLCP-Hedge assumes a uniformly feasible expert mixture, which may require conservative experts and can be restrictive in unbounded regression; it also uses full-information feedback and could add to computational burden. Future work could extend constrained aggregation to step sizes, localizers, distance metrics, calibration windows, and base predictors, and study learned localization, bandit-feedback variants, and stronger local coverage guarantees.


\bibliographystyle{plain}
\bibliography{ref}





\newpage
\appendix
\section{Proof of Proposition~\ref{prop:projected-olcp-boundary}}
\label{app:proof-projected-olcp}
\begin{proof}
 By definition,
\[
z_t=\alpha_t+\gamma(\alpha-\mathrm{err}_t),
\qquad
L_t=(-z_t)_+,
\qquad
U_t=(z_t-1)_+ .
\]
Since projection onto \([0,1]\) satisfies
\[
\Pi_{[0,1]}(z_t)=z_t+L_t-U_t,
\]
the projected update can be written as
\[
\alpha_{t+1}
=
\alpha_t+\gamma(\alpha-\mathrm{err}_t)+L_t-U_t .
\]
Summing this identity over \(t=1,\dots,T\) gives
\[
\alpha_{T+1}-\alpha_1
=
\gamma\sum_{t=1}^T(\alpha-\mathrm{err}_t)
+
\sum_{t=1}^T(L_t-U_t).
\]
Rearranging yields
\[
\sum_{t=1}^T(\mathrm{err}_t-\alpha)
=
\frac{\alpha_1-\alpha_{T+1}}{\gamma}
+
\frac{1}{\gamma}\sum_{t=1}^T(L_t-U_t),
\]
which proves the claimed identity.

Since \(\alpha_{T+1}\in[0,1]\), \(L_t\ge0\), and \(U_t\ge0\), we have
\[
\sum_{t=1}^T(\mathrm{err}_t-\alpha)
\le
\frac{\alpha_1}{\gamma}
+
\frac{1}{\gamma}\sum_{t=1}^T L_t .
\]
Dividing by \(T\) gives
\[
\frac1T\sum_{t=1}^T\mathrm{err}_t-\alpha
\le
\frac{\alpha_1}{T\gamma}
+
\frac{1}{T\gamma}\sum_{t=1}^T L_t .
\]
If \(\gamma>0\) is fixed and \(\sum_{t=1}^T L_t=o(T)\), the right-hand side converges to zero. Hence
\[
\limsup_{T\to\infty}
\left(
\frac1T\sum_{t=1}^T\mathrm{err}_t-\alpha
\right)
\le 0.
\]

For the lower deviation, the same identity gives
\[
\sum_{t=1}^T(\alpha-\mathrm{err}_t)
=
\frac{\alpha_{T+1}-\alpha_1}{\gamma}
+
\frac{1}{\gamma}\sum_{t=1}^T(U_t-L_t).
\]
Using \(\alpha_{T+1}\le1\), \(L_t\ge0\), and \(U_t\ge0\), we obtain
\[
\sum_{t=1}^T(\alpha-\mathrm{err}_t)
\le
\frac{1-\alpha_1}{\gamma}
+
\frac{1}{\gamma}\sum_{t=1}^T U_t .
\]
Dividing by \(T\),
\[
\alpha-\frac1T\sum_{t=1}^T\mathrm{err}_t
\le
\frac{1-\alpha_1}{T\gamma}
+
\frac{1}{T\gamma}\sum_{t=1}^T U_t .
\]
If \(\sum_{t=1}^T(L_t+U_t)=o(T)\), then in particular
\(\sum_{t=1}^T L_t=o(T)\) and \(\sum_{t=1}^T U_t=o(T)\). Therefore both the upper and lower deviations vanish, and hence
\[
\frac1T\sum_{t=1}^T\mathrm{err}_t\to \alpha .
\]
\end{proof}


\newcommand{\pfstep}[1]{\par\medskip\noindent\textbf{#1}\par\smallskip}

\section{More discussion on boundary terms}
\label{app:olcp-boundary}

This section expands on Proposition~\ref{prop:projected-olcp-boundary}. The proposition is an exact
pathwise identity for the projected update
\[
\alpha_{t+1}
=
\Pi_{[0,1]}\bigl(\alpha_t+\gamma(\alpha-\mathrm{err}_t)\bigr).
\]
The correction terms
\[
L_t=(-z_t)_+,\qquad U_t=(z_t-1)_+,
\qquad
z_t=\alpha_t+\gamma(\alpha-\mathrm{err}_t),
\]
measure how much the unconstrained update would have moved outside the interval \([0,1]\).
Thus \(L_t\) is the amount of lower-boundary clipping, and \(U_t\) is the amount of upper-boundary
clipping.

\paragraph{Interpretation.}
The term \(L_t>0\) can occur only when \(\mathrm{err}_t=1\) and
\[
\alpha_t<\gamma(1-\alpha).
\]
In words, OLCP misses while \(\alpha_t\) is already near \(0\). Since smaller \(\alpha_t\) corresponds
to wider prediction sets, \(L_t\) records the amount by which the update would like to widen the set
beyond the widest finite set allowed by the projected implementation. Similarly, \(U_t>0\) can occur
only when \(\mathrm{err}_t=0\) and
\[
\alpha_t>1-\gamma\alpha.
\]
This records the amount by which the update would like to shrink the set beyond the smallest finite
set allowed by the projection.

This is the key difference from the basic ACI telescoping guarantee. The original unprojected ACI
update allows \(\alpha_t\) to leave \([0,1]\), which can produce
full/infinite or empty prediction sets at the boundary. This makes the recursion telescope exactly
without correction terms. Projected OLCP instead keeps all prediction sets finite, so the price is the
appearance of \(L_t\) and \(U_t\).

\paragraph{A sufficient condition.}
We next give a simple sufficient condition under which the boundary corrections are negligible.
For clarity, allow the step size to depend on the horizon and write \(\gamma=\gamma_T\). Define
\[
\eta_T^-:=\gamma_T(1-\alpha),
\qquad
\eta_T^+:=\gamma_T\alpha .
\]
Since the prediction sets are decreasing in the nominal level, we have the deterministic inclusions
\[
\{L_t>0\}
\subseteq
\bigl\{Y_t\notin C_t^{(h)}(X_t;\eta_T^-)\bigr\},
\]
and
\[
\{U_t>0\}
\subseteq
\bigl\{Y_t\in C_t^{(h)}(X_t;1-\eta_T^+)\bigr\}.
\]
Moreover,
\[
0\le L_t
\le
\gamma_T(1-\alpha)\,
\mathbf 1\{Y_t\notin C_t^{(h)}(X_t;\eta_T^-)\},
\]
and
\[
0\le U_t
\le
\gamma_T\alpha\,
\mathbf 1\{Y_t\in C_t^{(h)}(X_t;1-\eta_T^+)\}.
\]
Therefore, a sufficient condition for
\[
\sum_{t=1}^T L_t=o(T\gamma_T)
\quad\text{and}\quad
\sum_{t=1}^T U_t=o(T\gamma_T)
\]
in expectation is
\[
\frac1T\sum_{t=1}^T
\mathbb P\!\left\{Y_t\notin C_t^{(h)}(X_t;\eta_T^-)\mid\mathcal F_{t-1}\right\}
=o(1),
\]
and
\[
\frac1T\sum_{t=1}^T
\mathbb P\!\left\{Y_t\in C_t^{(h)}(X_t;1-\eta_T^+)\mid\mathcal F_{t-1}\right\}
=o(1).
\]
These conditions say that boundary failures are rare: the near-widest finite set should not miss too
often, and the near-smallest finite set should not cover too often.

\paragraph{Empirical diagnostic.}
The boundary terms are directly observable during the run. In practice, one can report
\[
\frac{1}{T\gamma}\sum_{t=1}^T L_t
\qquad\text{and}\qquad
\frac{1}{T\gamma}\sum_{t=1}^T U_t
\]
as diagnostics. Small values indicate that projection is not materially affecting the long-run
coverage behavior; large values indicate that the algorithm is frequently hitting the finite-set
boundary, so the target coverage may not be attainable without wider or narrower endpoint sets.

\section{Feasibility diagnostics for assumption B}
\label{app:fea_dia}
Assumption~B requires a fixed mixture \(u^\star\in\Delta_K\) satisfying
\(\langle e_t,u^\star\rangle\le \alpha\) on every round. Since OLCP-Hedge uses full-information
feedback, this condition can be checked post hoc by the linear feasibility problem
\[
\text{find } u\in\Delta_K
\quad\text{such that}\quad
\langle e_t,u\rangle\le \alpha,\qquad t=1,\dots,T.
\]
Equivalently, one may solve the relaxed linear program
\[
\min_{u\in\Delta_K,\ \rho\ge0}\ \rho
\qquad
\text{subject to}
\qquad
\langle e_t,u\rangle-\alpha\le \rho,\quad t=1,\dots,T.
\]
If the optimum is \(\rho=0\), Assumption~B holds for the realized expert pool; otherwise \(\rho\)
quantifies the worst-round infeasibility of the best fixed mixture. When it's large, the uniform-feasibility comparator required by Theorem~\ref{thm:sv-olcp} is inappropriate, and a
relaxed online constraint-satisfaction benchmark such as \(S\)-feasibility or \(P_T\)-constrained
feasibility is more suitable.

\section{AdaHedge subroutine}
\label{app:adahedge}

We use AdaHedge as the black-box OCO subroutine over the probability simplex
\(\Delta_K\). Let
\[
  C_{\mathrm{AH}} := 2\sqrt{4+\ln K}.
\]

\begin{algorithm}[h]
\caption{AdaHedge on \(\Delta_K\) (FTRL form; \citep{orabona2019modern})}
\label{alg:adahedge-app}
\begin{algorithmic}[1]
\Require parameter \(\alpha_{\mathrm{AH}}>0\)
\State Initialize
\[
\lambda_1\gets 0,\qquad
p_1\gets (1/K,\dots,1/K)\in\Delta_K,\qquad
\theta_1\gets 0\in\mathbb R^K .
\]
\For{\(t=1,2,\dots,T\)}
  \State Output \(p_t\in\Delta_K\).
  \State Observe a linear loss vector \(\xi_t\in\mathbb R^K\) and incur loss
  \(\langle \xi_t,p_t\rangle\).
  \State Update \(\theta_{t+1}\gets \theta_t-\xi_t\).
  \If{\(t=1\)}
    \State \(\delta_t\gets \langle \xi_1,p_1\rangle-\min_{j\in[K]}\xi_{1,j}\).
  \Else
    \State
    \[
    \delta_t
    \gets
    \lambda_t
    \ln\!\left(
      \sum_{j=1}^K p_{t,j}\exp(-\xi_{t,j}/\lambda_t)
    \right)
    +
    \langle \xi_t,p_t\rangle .
    \]
  \EndIf
  \State Update \(\lambda_{t+1}\gets \lambda_t+\delta_t/\alpha_{\mathrm{AH}}^2\).
  \State Update
  \[
  p_{t+1,j}
  \propto
  \exp(\theta_{t+1,j}/\lambda_{t+1}),
  \qquad j=1,\dots,K .
  \]
\EndFor
\end{algorithmic}
\end{algorithm}

\begin{theorem}[AdaHedge bound on \(\Delta_K\)]
\label{thm:adahedge-app}
Run Algorithm~\ref{alg:adahedge-app} with
\[
\alpha_{\mathrm{AH}}=\sqrt{\ln K}.
\]
Then for any sequence \(\{\xi_t\}_{t=1}^T\subset\mathbb R^K\), the iterates
\(p_t\in\Delta_K\) satisfy, for all \(u\in\Delta_K\),
\[
  \sum_{t=1}^T \langle \xi_t,p_t-u\rangle
  \le
  C_{\mathrm{AH}}
  \sqrt{\sum_{t=1}^T \|\xi_t\|_\infty^2}.
\]
\end{theorem}

\begin{proof}
This is the parameter-free AdaHedge guarantee obtained by optimizing the standard
FTRL/AdaHedge bound; see \cite[Section~7]{orabona2019modern}. We use the
\(\ell_\infty\)-norm form because the decision set is the simplex and the losses are
linear in \(p\).
\end{proof}

\section{Proof of Theorem~\ref{thm:sv-olcp}}
\label{app:sv-proof}

\begin{proof}
We prove a slightly stronger prefix version: the stated bounds hold for every
\(m\le T\). The proof follows the Sinha--Vaze Lyapunov reduction for constrained online
convex optimization, specialized to the simplex \(\Delta_K\) and to the OLCP expert set size
and miscoverage constraint.

Recall that
\[
  f_t(p)=\langle \omega_t,p\rangle,
  \qquad
  g_t(p)=\langle e_t,p\rangle-\alpha,
\]
where \(\omega_t=(\omega_{t,1},\dots,\omega_{t,K})\) is the vector of expert sizes and
\(e_t=(\mathrm{err}_{t,1},\dots,\mathrm{err}_{t,K})\) is the vector of expert miscoverage indicators. We use the
preprocessed size loss and excess-miscoverage constraint
\[
  \tilde f_t(p):=\kappa f_t(p),
  \qquad
  \tilde g_t(p):=\kappa(g_t(p))_+ .
\]
The virtual queue is initialized at \(\mathcal Q(0)=0\) and updated as
\[
  \mathcal Q(t):=\mathcal Q(t-1)+\tilde g_t(p_t).
\]
The potential is
\[
  \Phi(q)=e^{\lambda q}-1,
  \qquad
  \Phi'(q)=\lambda e^{\lambda q},
\]
and the surrogate loss is
\[
  \hat f_t(p)
  :=
  V\tilde f_t(p)+\Phi'(\mathcal Q(t))\tilde g_t(p).
\]
In Theorem~\ref{thm:sv-olcp}, we set \(V=1\).

\pfstep{Step 1: Drift inequality and decomposition.}
Fix any prefix length \(m\le T\). By convexity of \(\Phi\) and the queue recursion
\[
\mathcal Q(t)=\mathcal Q(t-1)+\tilde g_t(p_t),
\]
we have
\[
  \Phi(\mathcal Q(t))-\Phi(\mathcal Q(t-1))
  \le
  \Phi'(\mathcal Q(t))\,\tilde g_t(p_t).
\]
Summing from \(t=1\) to \(m\) and using \(\Phi(\mathcal Q(0))=0\), we get
\[
  \Phi(\mathcal Q(m))
  \le
  \sum_{t=1}^m \Phi'(\mathcal Q(t))\,\tilde g_t(p_t).
\]
Let \(u^\star\in\Delta_K\) be any feasible comparator from Assumption~B. Then
\[
g_t(u^\star)\le 0
\qquad\Rightarrow\qquad
(g_t(u^\star))_+=0
\qquad\Rightarrow\qquad
\tilde g_t(u^\star)=0 .
\]
Since \(V=1\),
\[
\begin{aligned}
  \sum_{t=1}^m
  \bigl(\hat f_t(p_t)-\hat f_t(u^\star)\bigr)
  &=
  \sum_{t=1}^m
  \Bigl(\tilde f_t(p_t)-\tilde f_t(u^\star)\Bigr)
  +
  \sum_{t=1}^m
  \Phi'(\mathcal Q(t))\,\tilde g_t(p_t).
\end{aligned}
\]
Therefore,
\begin{equation}
  \Phi(\mathcal Q(m))
  +
  \sum_{t=1}^m
  \Bigl(\tilde f_t(p_t)-\tilde f_t(u^\star)\Bigr)
  \le
  \sum_{t=1}^m
  \bigl(\hat f_t(p_t)-\hat f_t(u^\star)\bigr).
  \label{eq:sv-decomp-olcp-app}
\end{equation}

\pfstep{Step 2: Apply AdaHedge to the linearized surrogate losses.}
For each \(t\), choose any
\[
  \xi_t\in\partial \hat f_t(p_t).
\]
By convexity,
\[
  \hat f_t(p_t)-\hat f_t(u^\star)
  \le
  \langle \xi_t,p_t-u^\star\rangle .
\]
Thus, by Theorem~\ref{thm:adahedge-app} applied to the prefix \(1,\dots,m\),
\begin{equation}
  \sum_{t=1}^m
  \bigl(\hat f_t(p_t)-\hat f_t(u^\star)\bigr)
  \le
  C_{\mathrm{AH}}
  \sqrt{\sum_{t=1}^m \|\xi_t\|_\infty^2}.
  \label{eq:ah-apply-olcp-app}
\end{equation}

\pfstep{Step 3: Bound the surrogate gradients.}
Since
\[
  \hat f_t(p)
  =
  \tilde f_t(p)+\Phi'(\mathcal Q(t))\tilde g_t(p),
\]
and
\[
  \tilde f_t(p)=\kappa f_t(p),
  \qquad
  \tilde g_t(p)=\kappa(g_t(p))_+,
\]
Assumption~A implies that we may choose \(\xi_t\in\partial \hat f_t(p_t)\) such that
\[
  \|\xi_t\|_\infty
  \le
  \kappa G\bigl(1+\Phi'(\mathcal Q(t))\bigr).
\]
Hence, using \((a+b)^2\le 2(a^2+b^2)\),
\[
  \sum_{t=1}^m \|\xi_t\|_\infty^2
  \le
  2\kappa^2G^2
  \left(
    m+\sum_{t=1}^m(\Phi'(\mathcal Q(t)))^2
  \right).
\]
Because \(\mathcal Q(t)\) is nondecreasing and \(\Phi'\) is nondecreasing,
\[
  \Phi'(\mathcal Q(t))\le \Phi'(\mathcal Q(m)),
  \qquad t\le m.
\]
Therefore,
\[
  \sum_{t=1}^m(\Phi'(\mathcal Q(t)))^2
  \le
  m(\Phi'(\mathcal Q(m)))^2.
\]
Plugging this into \eqref{eq:ah-apply-olcp-app} gives
\begin{equation}
  \sum_{t=1}^m
  \bigl(\hat f_t(p_t)-\hat f_t(u^\star)\bigr)
  \le
  \sqrt{2}\,C_{\mathrm{AH}}\kappa G\sqrt m
  \bigl(1+\Phi'(\mathcal Q(m))\bigr).
  \label{eq:surrogate-reg-olcp-app}
\end{equation}

\pfstep{Step 4: Size-regret bound.}
Combining \eqref{eq:sv-decomp-olcp-app} and
\eqref{eq:surrogate-reg-olcp-app}, and using
\[
\Phi(q)=e^{\lambda q}-1,
\qquad
\Phi'(q)=\lambda e^{\lambda q},
\]
yields
\[
  e^{\lambda \mathcal Q(m)}-1
  +
  \sum_{t=1}^m
  \bigl(\tilde f_t(p_t)-\tilde f_t(u^\star)\bigr)
  \le
  \sqrt{2}\,C_{\mathrm{AH}}\kappa G\sqrt m
  \bigl(1+\lambda e^{\lambda \mathcal Q(m)}\bigr).
\]
With
\[
  \kappa=(\sqrt{2}C_{\mathrm{AH}}G)^{-1},
\]
the prefactor is equal to \(1\), so
\[
  e^{\lambda \mathcal Q(m)}-1
  +
  \sum_{t=1}^m
  \bigl(\tilde f_t(p_t)-\tilde f_t(u^\star)\bigr)
  \le
  \sqrt m+\lambda\sqrt m\,e^{\lambda \mathcal Q(m)}.
\]
Rearranging,
\[
  \sum_{t=1}^m
  \bigl(\tilde f_t(p_t)-\tilde f_t(u^\star)\bigr)
  \le
  \sqrt m+1+(\lambda\sqrt m-1)e^{\lambda \mathcal Q(m)}.
\]
With
\[
  \lambda=\frac{1}{2\sqrt T}
  \qquad\text{and}\qquad
  m\le T,
\]
we have \(\lambda\sqrt m\le 1/2\), so the last term is nonpositive. Therefore,
\[
  \sum_{t=1}^m
  \bigl(\tilde f_t(p_t)-\tilde f_t(u^\star)\bigr)
  \le
  \sqrt m+1
  \le
  2\sqrt m .
\]
Since \(\tilde f_t=\kappa f_t\),
\[
  \sum_{t=1}^m
  \bigl(f_t(p_t)-f_t(u^\star)\bigr)
  \le
  2\kappa^{-1}\sqrt m.
\]
Using
\[
  \kappa^{-1}=\sqrt{2}\,C_{\mathrm{AH}}G
  =
  2G\sqrt{2(4+\ln K)},
\]
we obtain
\[
  \sum_{t=1}^m
  \bigl(f_t(p_t)-f_t(u^\star)\bigr)
  \le
  4G\sqrt{2(4+\ln K)\,m}.
\]
Finally, since
\[
  f_t(p_t)=\mathbb E_{I_t\sim p_t}[\omega_{t,I_t}],
  \qquad
  f_t(u^\star)=\langle \omega_t,u^\star\rangle,
\]
we have, for every \(m\le T\),
\[
  \sum_{t=1}^m
  \Bigl(
    \mathbb E_{I_t\sim p_t}[\omega_{t, I_t}]
    -
    \langle \omega_t,u^\star\rangle
  \Bigr)
  \le
  4G\sqrt{2(4+\ln K)\,m}.
\]

\pfstep{Step 5: Cumulative excess-miscoverage bound.}
By Assumption~A and the fact that the \(\ell_1\)-diameter of \(\Delta_K\) is \(2\),
\[
  \tilde f_t(p_t)-\tilde f_t(u^\star)
  \ge
  -2\kappa G.
\]
Thus,
\[
  \sum_{t=1}^m
  \bigl(\tilde f_t(p_t)-\tilde f_t(u^\star)\bigr)
  \ge
  -2\kappa Gm.
\]
Plugging this lower bound into the inequality from Step 4 before dropping the exponential
term gives
\[
  e^{\lambda \mathcal Q(m)}(1-\lambda\sqrt m)
  \le
  1+\sqrt m+2\kappa Gm.
\]
Therefore,
\[
  \mathcal Q(m)
  \le
  \frac{1}{\lambda}
  \ln\!\left(
    \frac{1+\sqrt m+2\kappa Gm}{1-\lambda\sqrt m}
  \right).
\]
For \(\lambda=1/(2\sqrt T)\) and \(m\le T\),
\[
1-\lambda\sqrt m\ge \frac12,
\]
so
\[
  \mathcal Q(m)
  \le
  \frac{1}{\lambda}
  \ln\!\Bigl(2(1+\sqrt m+2\kappa Gm)\Bigr)
  \le
  \frac{1}{\lambda}
  \ln\!\Bigl(2+(2+4\kappa G)m\Bigr),
\]
where we used \(1+\sqrt m\le 1+m\). Since
\[
  \kappa G=\frac{1}{\sqrt{2}C_{\mathrm{AH}}},
  \qquad
  C_{\mathrm{AH}}=2\sqrt{4+\ln K}\ge 4,
\]
we have
\[
  4\kappa G
  =
  \frac{4}{\sqrt{2}C_{\mathrm{AH}}}
  \le
  \frac{\sqrt 2}{2}.
\]
Hence,
\[
  \mathcal Q(m)
  \le
  \frac{1}{\lambda}
  \ln\!\Bigl(
    2+\bigl(2+\tfrac{\sqrt2}{2}\bigr)m
  \Bigr).
\]
Finally, since
\[
  \mathcal Q(m)
  =
  \sum_{t=1}^m \kappa(g_t(p_t))_+,
\]
we obtain
\[
  \sum_{t=1}^m(g_t(p_t))_+
  =
  \kappa^{-1}\mathcal Q(m)
  \le
  \frac{\kappa^{-1}}{\lambda}
  \ln\!\Bigl(
    2+\bigl(2+\tfrac{\sqrt2}{2}\bigr)m
  \Bigr).
\]
Using
\[
  \frac{\kappa^{-1}}{\lambda}
  =
  4G\sqrt{2(4+\ln K)\,T},
\]
we get, for every \(m\le T\),
\[
  \sum_{t=1}^m(g_t(p_t))_+
  \le
  4G\sqrt{2(4+\ln K)\,T}
  \ln\!\Bigl(
    2+\bigl(2+\tfrac{\sqrt2}{2}\bigr)m
  \Bigr).
\]
Since
\[
  g_t(p_t)
  =
  \mathbb E_{I_t\sim p_t}[\mathrm{err}_{t, I_t}]-\alpha,
\]
this proves
\[
  \sum_{t=1}^m
  \Bigl(
    \mathbb E_{I_t\sim p_t}[\mathrm{err}_{t, I_t}]-\alpha
  \Bigr)_+
  \le
  4G\sqrt{2(4+\ln K)\,T}
  \ln\!\Bigl(
    2+\bigl(2+\tfrac{\sqrt2}{2}\bigr)m
  \Bigr).
\]
The proof is complete.
\end{proof}

\section{Experiments}
\label{app:exp}

\subsection{Methods}
\label{app:method}
This section describes the implementation used in the synthetic and real-data experiments. All
methods are implemented as online wrappers around a base predictor.  In the regression experiments,
\[
S_t = |Y_t-\hat Y_t|.
\]

\paragraph{Online calibration windows.}
Each series is processed independently. Let
\[
\mathcal T_{\mathrm{test}}=(\tau_1,\dots,\tau_{T_{\mathrm{test}}})
\]
be the ordered test indices for a given series. At test position \(j\), the calibration window contains
the previous at most \(R\) test-time observations:
\[
\mathcal I_j
=
\{\tau_{\max(1,j-R)},\dots,\tau_{j-1}\}.
\]
The first test point with an empty calibration window is skipped. This convention ensures that all
calibration is online and uses no future outcomes. 

\paragraph{CP.}
CP uses the finite-sample corrected rolling quantile of the previous calibration scores. If
\(r=|\mathcal I_j|\), then
\[
k=\left\lceil (1-\alpha)(r+1)\right\rceil,
\]
with \(k\) clipped to \(\{1,\dots,r\}\), and
\[
q_t^{\text{CP}}=\text{the }k\text{-th order statistic of }\{S_i:i\in\mathcal I_j\}.
\]
The prediction set is \([\hat Y_t-q_t^{\text{CP}},\hat Y_t+q_t^{\text{CP}}]\).

\paragraph{Localized weighted quantiles.}
LCP and OLCP use weighted empirical quantiles. For a query covariate \(x\), the calibration
covariates in \(\mathcal I_j\) are standardized coordinatewise using their empirical mean and standard
deviation:
\[
Z_i=\frac{X_i-\bar X_{\mathcal I_j}}{\hat\sigma_{\mathcal I_j}},
\qquad
Z_x=\frac{x-\bar X_{\mathcal I_j}}{\hat\sigma_{\mathcal I_j}},
\]
where zero or numerically unstable standard deviations are replaced by \(1\). We then use the
exponential localizer
\[
H_h(x,X_i)=\exp\!\left(-\frac{\|Z_i-Z_x\|_2}{h}\right).
\]
The weighted quantile is computed by sorting calibration scores and accumulating the corresponding
sorted weights. If the total weight is numerically zero, the implementation falls back to uniform
weights.

The base bandwidth is chosen by a Silverman-style rule,
\[
h_0
=
\left(\frac{4}{d+2}\right)^{1/(d+4)}
R^{-1/(d+4)}
\sqrt d,
\]
where \(d\) is the covariate dimension. The factor
\(\left(4/(d+2)\right)^{1/(d+4)}R^{-1/(d+4)}\) is the classical multivariate kernel bandwidth scaling
\citep{silverman1986density,scott1992multivariate}. Because covariates are standardized within each
calibration window before computing Euclidean distances, the additional \(\sqrt d\) factor matches the
typical scale of distances in \(d\) standardized dimensions. OLCP-Hedge further reduces sensitivity
to this heuristic by aggregating a grid of multiplicative bandwidths around \(h_0\).

\paragraph{LCP.}
LCP uses the localized empirical score distribution at the current covariate \(X_t\),
\[
D_t^{(h_0)}(X_t)
=
\sum_{i\in\mathcal I_j}
w_{t,i}^{(h_0)}(X_t)\,\delta_{S_i},
\qquad
w_{t,i}^{(h_0)}(X_t)
=
\frac{H_{h_0}(X_t,X_i)}
{\sum_{r\in\mathcal I_j}H_{h_0}(X_t,X_r)} .
\]
The fixed-level localized radius is
\[
q_t^{\text{LCP}}
=
Q\!\left(1-\alpha;\,D_t^{(h_0)}(X_t)\right),
\]
and the prediction set is
\[
[\hat Y_t-q_t^{\text{LCP}},\ \hat Y_t+q_t^{\text{LCP}}].
\]

\paragraph{ACI.}
ACI uses the rolling unweighted empirical score distribution, but replaces the fixed nominal level
\(\alpha\) by an adaptive level \(\alpha_t\). For a test stream of length \(T_{\mathrm{test}}\), the default
step size is
\[
\gamma=\frac{1}{2\sqrt{T_{\mathrm{test}}}}.
\]
At time \(t\), ACI forms the rolling quantile at level \(1-\alpha_t\). After observing \(Y_t\), it updates
\[
\alpha_{t+1}
=
\Pi_{[0,1]}\bigl(\alpha_t+\gamma(\alpha-\mathrm{err}_t)\bigr),
\qquad
\mathrm{err}_t=\mathbf 1\{Y_t\notin \widehat C_t\},
\]
with initialization \(\alpha_1=\alpha\). Here we project the 
$\alpha_t$ back to $[0,1]$ in the experiment to prevent infinite prediction sets for a better comparison.

\paragraph{DtACI.}
DtACI aggregates multiple ACI experts with different step sizes. Let
\[
\gamma_0=\frac{1}{2\sqrt{T_{\mathrm{test}}}},
\qquad
\Gamma=\{0.25,0.5,0.75,1,1.25,1.5\}\gamma_0.
\]
Each expert \(r\) maintains its own adaptive level \(\alpha_t^{(r)}\). The mixture level is
\[
\alpha_t^{\mathrm{DtACI}}=\sum_r p_{t,r}\alpha_t^{(r)}.
\]
At time \(t\), the method computes the empirical rank statistic
\[
\beta_t
=
\frac{1}{|\mathcal I_j|}
\sum_{i\in\mathcal I_j}\mathbf 1\{S_i\ge S_t\}.
\]
Expert weights are updated using exponential weights on the pinball loss
\[
\ell(\beta_t,\alpha_t^{(r)})
=
\alpha(\beta_t-\alpha_t^{(r)})-\min\{0,\beta_t-\alpha_t^{(r)}\}.
\]
Following \cite{gibbs2024conformal}, the implementation uses
\[
I_{\mathrm{size}}=500,
\qquad
\sigma_{\mathrm{dt}}=\frac{1}{2I_{\mathrm{size}}},
\]
and
\[
\eta_{\mathrm{dt}}
=
\sqrt{\frac{3}{I_{\mathrm{size}}}}\,
\sqrt{
\frac{\log(I_{\mathrm{size}}|\Gamma|)+2}
{
\bigl((1-\alpha)^2\alpha^3+\alpha^2(1-\alpha)^3\bigr)/3
}
}.
\]
The expert levels are then updated by their own projected ACI recursions.

\paragraph{SPCI.}
SPCI is implemented as a residual-forecasting conformal baseline. Let
\[
\varepsilon_t=Y_t-\hat Y_t
\]
denote the base-model residual on the online test stream. For each test time, we form a lag vector of
the previous \(w_{\mathrm{lag}}\) residuals and fit a quantile random forest to predict the next residual.
The default parameters are
\[
w_{\mathrm{lag}}=24,
\qquad
T_{\mathrm{train}}=R,
\qquad
\texttt{refit\_every}=24,
\qquad
\texttt{beta\_grid\_size}=21.
\]
The quantile random forest uses
\[
\texttt{n\_estimators}=80,\quad
\texttt{max\_depth}=10,\quad
\texttt{min\_samples\_leaf}=5,\quad
\texttt{random\_state}=42,\quad
\texttt{n\_jobs}=-1.
\]
For a grid
\[
\beta\in\left\{0,\frac{\alpha}{20},\frac{2\alpha}{20},\dots,\alpha\right\},
\]
SPCI predicts residual quantiles at levels \(\beta\) and \(1-\alpha+\beta\), and selects the value
of \(\beta\) giving the narrowest prediction set. The final prediction set is
\[
[\hat Y_t+\hat q_{\beta},\ \hat Y_t+\hat q_{1-\alpha+\beta}].
\]

\paragraph{OLCP.}
OLCP uses the localized family \(C_t^{(h)}(X_t;\beta)\) from Section~\ref{sec:olcp}. In the
experiments, \(h=h_0\). At time \(t\), OLCP forms
\[
\widehat C_t
=
C_t^{(h_0)}(X_t;\alpha_t),
\]
where the quantile is the localized weighted quantile at level \(1-\alpha_t\). After observing \(Y_t\),
OLCP updates
\[
\alpha_{t+1}
=
\Pi_{[0,1]}\bigl(\alpha_t+\gamma(\alpha-\mathrm{err}_t)\bigr),
\qquad
\mathrm{err}_t=\mathbf 1\{Y_t\notin \widehat C_t\},
\]
with
\[
\gamma=\frac{1}{2\sqrt{T_{\mathrm{test}}}}.
\]

\paragraph{OLCP-Hedge.}
OLCP-Hedge aggregates OLCP experts over the bandwidth grid
\[
h_i\in\{0.5,0.75,1,1.25,1.5\}h_0,
\qquad i=1,\dots,5.
\]
Expert \(i\) maintains its own adaptive level \(\alpha_{t,i}\) and outputs
\[
C_{t,i}(X_t)=C_t^{(h_i)}(X_t;\alpha_{t,i}).
\]
In regression, the set-size functional is interval width, so
\[
\omega_{t,i}=\operatorname{width}(C_{t,i}(X_t)),
\]
and the expert miscoverage indicator is
\[
e_{t,i}:=\mathbf 1\{Y_t\notin C_{t,i}(X_t)\}.
\]
Each expert level is updated by
\[
\alpha_{t+1,i}
=
\Pi_{[0,1]}\bigl(\alpha_{t,i}+\gamma(\alpha-e_{t,i})\bigr),
\]
with
\[
\gamma=\frac{1}{2\sqrt{T_{\mathrm{test}}}}.
\]

At round $t$, the expert cost is the prediction set size $\omega_{t, i}$, min-max normalized to $[0,1]$ across experts, and all other parameters are chosen according to Section~\ref{sec:sv-olcp-adahedge} with $G = 1$.

\subsection{Running time on experiments}
\label{app:running_time}

All experiments were run on a MacBook Pro. Neural network predictors were trained using Apple's MPS backend when available. The conformal runtime table was measured on the same machine and reports only the online conformal calibration/evaluation step; it excludes data loading, base predictor training, forecast precomputation.

\begin{table}[h]
\centering
\caption{Running time comparison. Entries report wall-clock time in seconds for the conformal calibration/evaluation step only. For simulation, we report the mean running time across 100 repetitions.}
\label{tab:runtime}
\small
\setlength{\tabcolsep}{4.5pt}
\begin{tabular}{lcccccc}
\toprule
Method
& Sim. A
& Sim. B
& Sim. C
& ELEC2
& ILINet
& ETF volatility \\
\midrule
CP
& 0.02 & 0.02 & 0.02 & 0.04 & 0.01 & 0.13 \\
LCP
& 0.04 & 0.04 & 0.04 & 0.21 & 0.01 & 0.54 \\
ACI
& 0.01 & 0.01 & 0.01 & 0.04 & 0.01 & 0.10 \\
DtACI
& 0.02 & 0.02 & 0.02 & 0.11 & 0.01 & 0.20 \\
SPCI
& 19.51 & 19.52 & 19.52 & 136.91 & 16.79 & 208.10 \\
\midrule
OLCP
& 0.05 & 0.05 & 0.05 & 0.26 & 0.01 & 0.41 \\
OLCP-Hedge
& 0.09 & 0.09 & 0.09 & 0.60 & 0.04 & 0.75 \\
\bottomrule
\end{tabular}
\end{table}

\subsection{Additional details and diagnostics for real-data experiments}
\label{app:real-data}

This section provides implementation details and additional diagnostics for the real-data experiments
in Section~\ref{sec:realdata}. We evaluate the same seven methods as in the synthetic experiments:
CP, LCP, ACI, DtACI, SPCI, OLCP, and OLCP-Hedge. All methods are evaluated online with target
miscoverage \(\alpha=0.1\) using rolling calibration windows. Reported sizes correspond to interval
width for ELEC2 and ILINet; for ETF volatility, sizes are multiplied by \(100\) and reported in
percentage points of absolute log return.

\paragraph{Implementation details.}
All experiments use a fixed point predictor followed by online conformal calibration. The point
predictor is trained only on the training split and is not updated during conformal evaluation. The
conformal methods use symmetric intervals of the form
\[
[\widehat y_t-q_t,\widehat y_t+q_t],
\]
with conformity score \(S_t=|Y_t-\widehat y_t|\). SPCI is implemented as a residual-autoregressive
baseline using a sliding training window comparable to the conformal calibration window.

\begin{itemize}
    \item \textbf{ELEC2.}
    ELEC2 contains electricity market prices, demands, and transfers from New South Wales and
    Victoria \citep{harries1999splice}. We use the normalized ELEC2 file and remove the initial
    constant-response segment, leaving \(27{,}552\) observations. The response is electricity transfer,
    and the covariates are \texttt{nswprice}, \texttt{nswdemand}, \texttt{vicprice}, and
    \texttt{vicdemand}. We keep the full half-hourly sequence. The first \(70\%\) of observations are
    used to train a fixed gradient-boosted regression tree predictor, implemented as
    \texttt{HistGradientBoostingRegressor} with maximum depth \(6\), learning rate \(0.05\), \(400\)
    boosting iterations, and random seed \(42\). Conformal methods are
    evaluated on the remaining \(30\%\) of the sequence with calibration window \(R=100\). For
    localized methods, the localization feature is the four-dimensional covariate vector above. For
    SPCI, we use residual lag length \(24\), training window \(R\), refit frequency \(24\), and \(21\)
    candidate \(\beta\)-values.

    \item \textbf{ILINet.}
    ILINet is a weekly CDC influenza-like illness surveillance dataset \citep{cdcFluView,darts}. We use
    the weighted ILI component, which contains \(1{,}305\) weekly observations from October 1997 to
    October 2022. Missing values are filled by interpolation followed by forward/backward filling.
    The series is split chronologically into \(70\%\) training, \(10\%\) validation, and \(20\%\) testing,
    giving \(913\), \(130\), and \(262\) observations, respectively. The response is standardized using
    the training mean and standard deviation, and intervals are constructed on this standardized scale.
    The base predictor is a temporal convolutional network (TCN)
    \citep{lea2016temporalconvolutionalnetworksunified} with input length \(26\), output length \(1\),
    batch size \(32\), at most \(200\) epochs, kernel size \(5\), \(8\) filters, dilation base \(2\),
    dropout \(0.2\), and Adam learning rate \(10^{-3}\). We use early stopping on validation loss with
    patience \(3\), minimum improvement \(10^{-3}\), gradient clipping at \(0.1\), and the best
    checkpoint. One-step forecasts are computed by historical forecasting with
    \texttt{forecast\_horizon=1}, \texttt{stride=1}, and no retraining. For localized methods, \(X_t\)
    is the lag window of the previous \(26\) standardized ILI values. The calibration window is
    \(R=52\). For SPCI, we use residual lag length \(8\), training window \(R\), refit frequency \(1\),
    and \(21\) candidate \(\beta\)-values.

    \item \textbf{ETF volatility.}
    We forecast daily volatility proxies for five ETFs: SPY, QQQ, IWM, EEM, and TLT \citep{etf}.
    Daily closing prices are read from downloaded Stooq files, and the VIX index is obtained from
    FRED/CBOE \citep{cboeVIX}. The data are aligned to a business-day grid by forward-filling prices
    and VIX values. The response is the absolute daily log return
    \[
        Y_t = |\log P_t-\log P_{t-1}|,
    \]
    and the sample runs from January 2008 to March 2026. Each ETF series has \(4{,}742\)
    business-day observations after alignment. We use a chronological split with training through
    2018, validation over 2019, and testing from 2020 onward, giving \(2{,}868\), \(261\), and
    \(1{,}613\) observations per asset. Each ETF volatility series is standardized using its own
    training mean and standard deviation; VIX is standardized using the VIX training split. The base
    predictor is a TCN trained jointly on the five standardized ETF volatility series, with input length
    \(30\), output length \(1\), batch size \(256\), at most \(80\) epochs, kernel size \(5\), \(8\) filters,
    dilation base \(2\), dropout \(0.2\), and Adam learning rate \(10^{-3}\). We include cyclic calendar
    encoders for day of week and month. Early stopping monitors validation loss with patience \(5\),
    minimum improvement \(10^{-3}\), gradient clipping at \(0.1\), and the best checkpoint. For
    localized methods, \(X_t\) consists of the previous \(30\) standardized volatility values
    concatenated with the lagged standardized VIX value, so the localization dimension is
    \(31\). The calibration window is \(R=200\). Since the TCN is trained on standardized responses,
    widths are converted back to the original absolute-log-return scale using the asset-specific
    training standard deviation and reported as \(100\times\) width. For SPCI, we use residual lag
    length \(30\), training window \(R\), refit frequency \(5\), and \(21\) candidate \(\beta\)-values.
\end{itemize}

\paragraph{Rolling diagnostics.}
Figures~\ref{fig:elec2-roll}--\ref{fig:etf-roll} show rolling coverage and rolling average size on the
three real datasets. In each figure, the top panel shows rolling coverage and the bottom panel shows
rolling average size. A desirable method stays close to the dashed \(0.90\) coverage line in the top
panel while having a lower curve in the bottom panel.

\begin{figure}[H]
  \centering
  \includegraphics[width=0.95\linewidth]{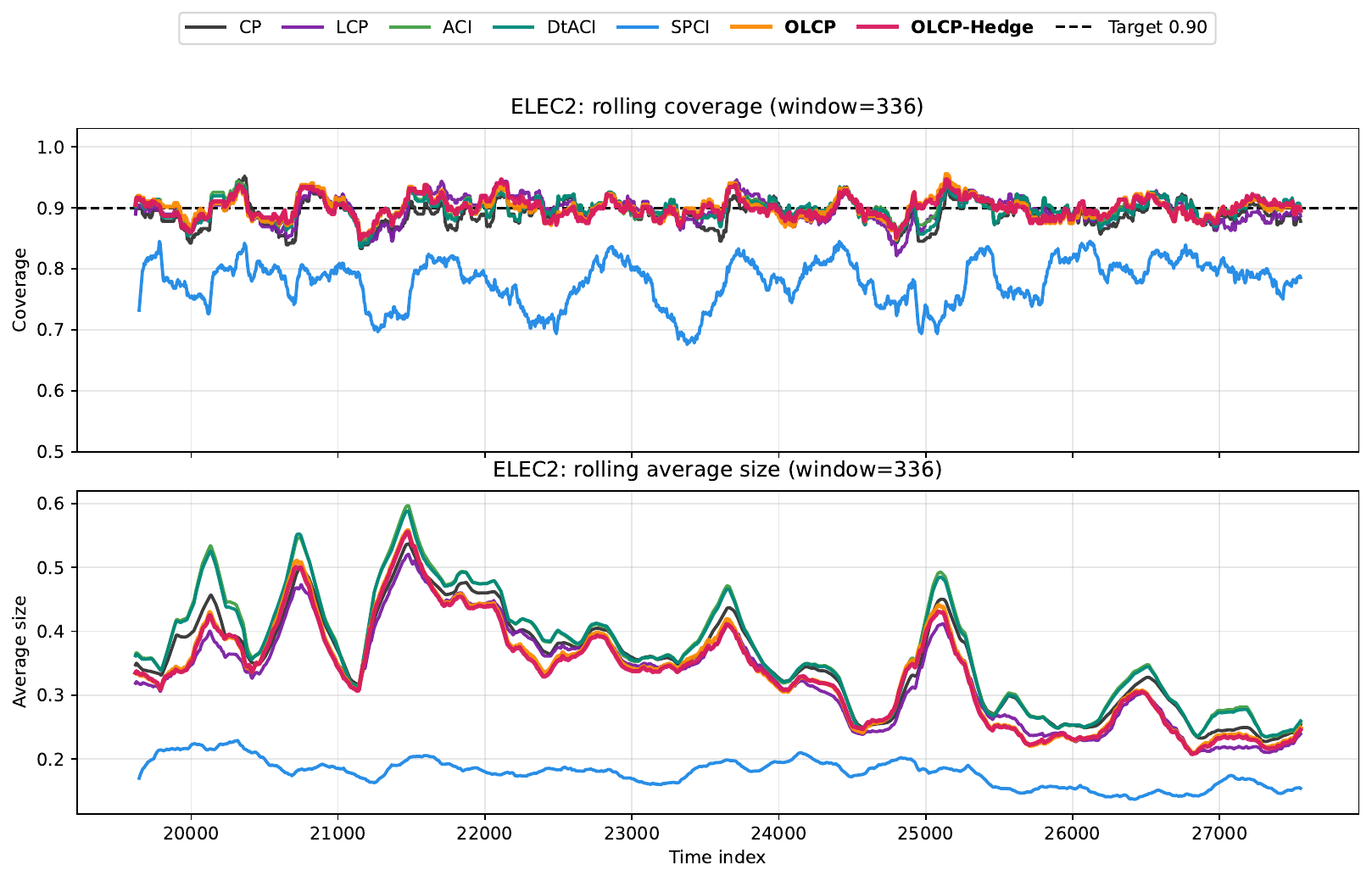}
  \caption{\textbf{ELEC2 rolling diagnostics.}
  Top: rolling coverage using a one-week window (\(48\times 7\) half-hourly observations).
  Bottom: rolling average interval size using the same window. The dashed line marks target coverage
  \(0.90\).}
  \label{fig:elec2-roll}
\end{figure}

\begin{figure}[H]
  \centering
  \includegraphics[width=0.95\linewidth]{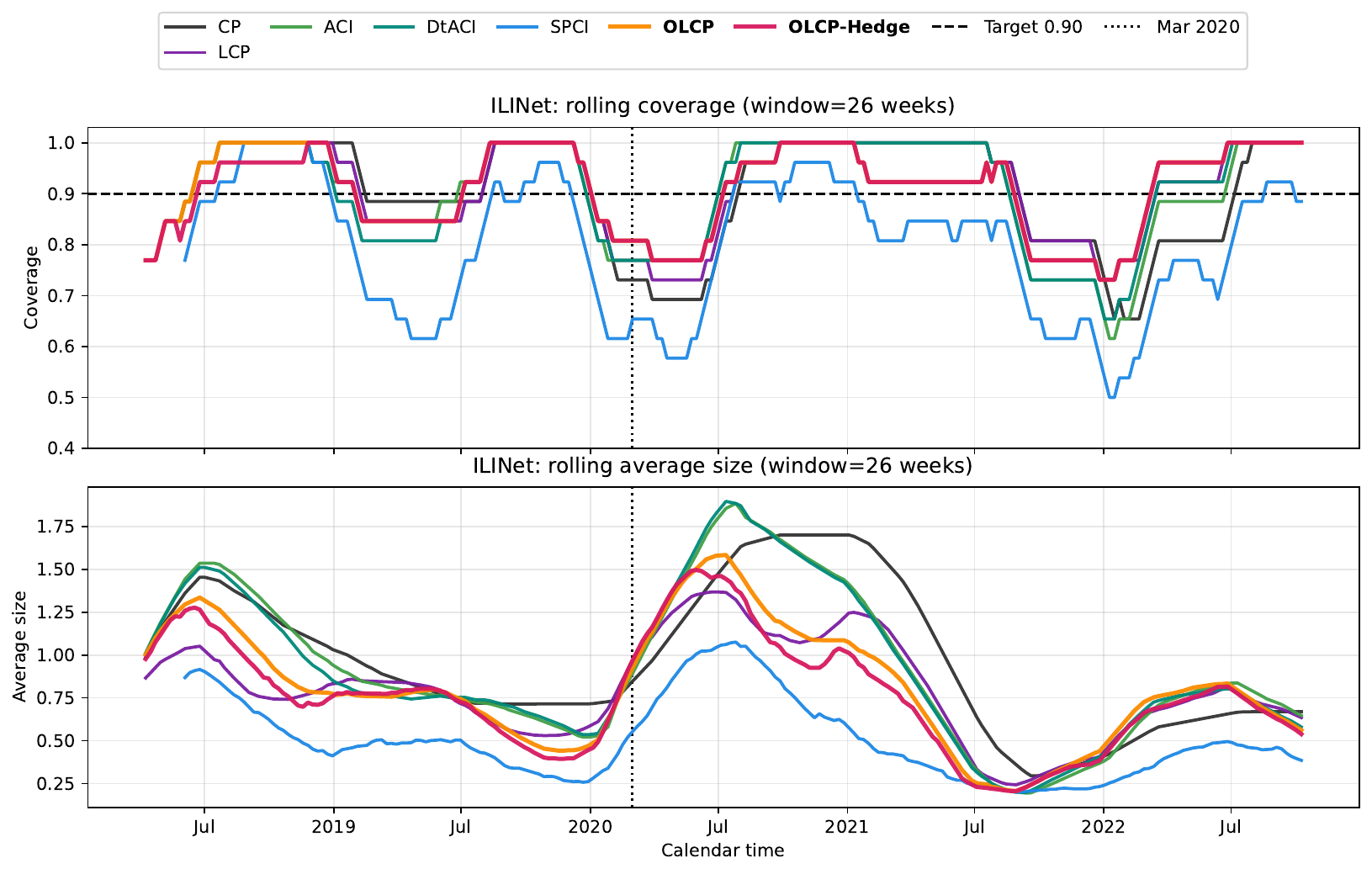}
  \caption{\textbf{ILINet rolling diagnostics.}
  Top: rolling coverage over weekly test observations. Bottom: rolling average interval size. The horizontal dashed
  line marks target coverage \(0.90\) while the vertical line marks the start of COVID.}
  \label{fig:ili-roll}
\end{figure}

\begin{figure}[H]
  \centering
  \includegraphics[width=0.95\linewidth]{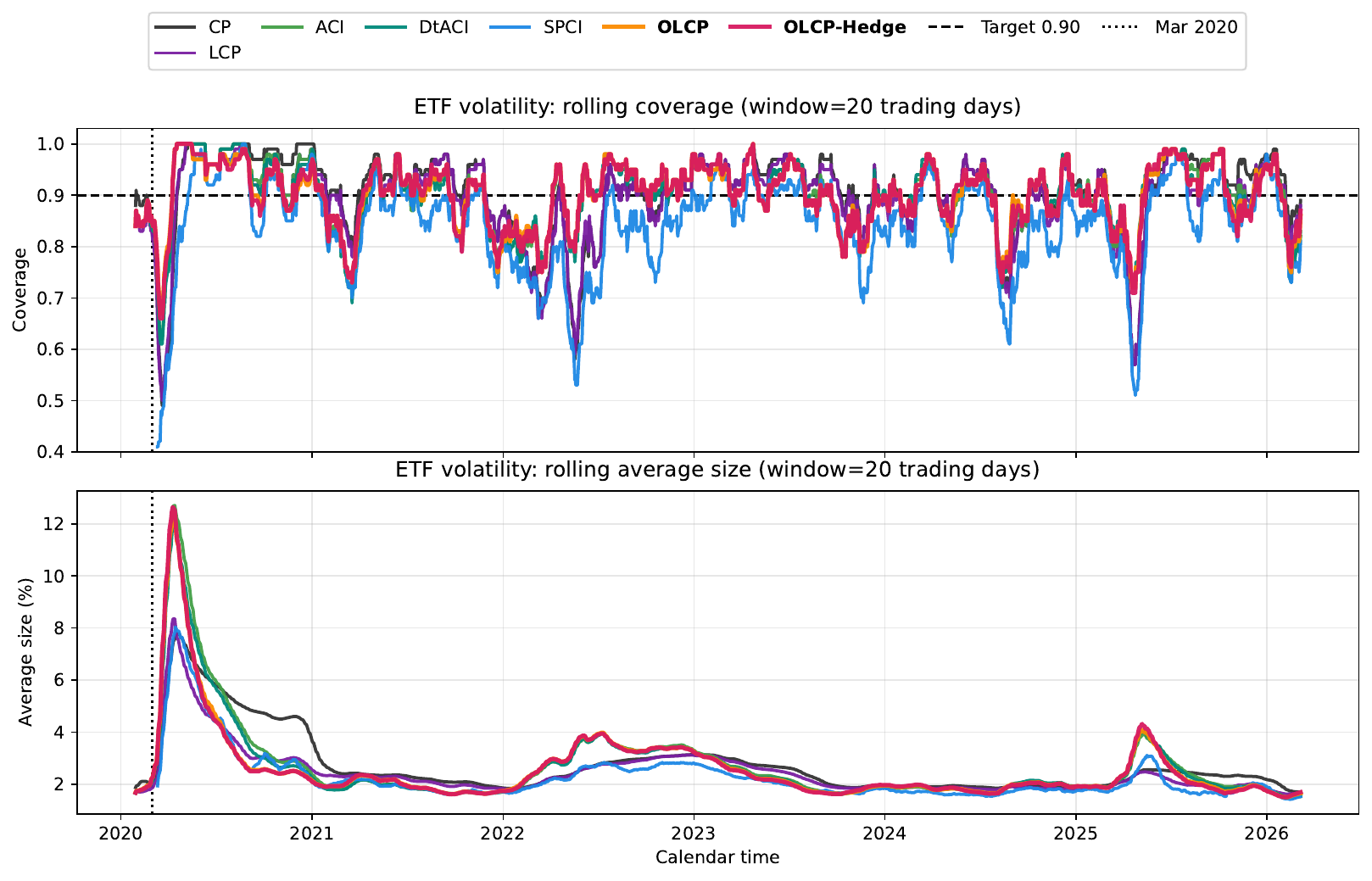}
  \caption{\textbf{ETF volatility rolling diagnostics.}
  Top: rolling coverage over daily test observations. Bottom: rolling average interval size, reported in
  percentage points of absolute log return. The dashed line marks target coverage \(0.90\).}
  \label{fig:etf-roll}
\end{figure}

Across the rolling diagnostics, SPCI is consistently much smaller but also persistently below the
coverage target, indicating that its residual autoregression is not sufficiently calibrated in these
nonstationary streams. ACI and DtACI tend to recover coverage by increasing sizes globally. OLCP
and OLCP-Hedge are more efficient: their rolling sizes are generally below the global adaptive
baselines while their rolling coverage remains close to the target, especially outside the most extreme
stress periods.

\paragraph{Volatility-regime diagnostics for ETF data.}
Table~\ref{tab:vix-regime} stratifies ETF volatility performance by the current VIX level. This
diagnostic evaluates whether methods adapt to market-volatility regimes rather than only achieving
marginal coverage.

\begin{table}[h]
\centering
\caption{\textbf{ETF volatility diagnostics by VIX regime.}
Low- and high-VIX regimes are defined as the bottom and top quartiles of VIX over the online test
period: \(\mathrm{VIX}\le 15.87\) and \(\mathrm{VIX}\ge 23.84\), respectively. Sizes are reported in
percentage points of absolute log return. \(N\) is the number of evaluated ETF-day prediction points
in each regime; SPCI has fewer low-VIX points because of its additional residual-lag warm-up.}
\label{tab:vix-regime}
\small
\setlength{\tabcolsep}{4.5pt}
\begin{tabular}{lccc ccc}
\toprule
& \multicolumn{3}{c}{Low VIX} & \multicolumn{3}{c}{High VIX} \\
\cmidrule(lr){2-4}\cmidrule(lr){5-7}
Method & Coverage & Size & \(N\) & Coverage & Size & \(N\) \\
\midrule
CP          & 0.925 & 2.044 & 2015 & 0.854 & 3.835 & 2015 \\
LCP         & 0.920 & 1.900 & 2015 & 0.846 & 3.356 & 2015 \\
ACI         & 0.913 & 1.851 & 2015 & 0.890 & 4.349 & 2015 \\
DtACI       & 0.915 & 1.854 & 2015 & 0.889 & 4.184 & 2015 \\
SPCI        & 0.862 & 1.693 & 1895 & 0.792 & 3.272 & 2015 \\
\midrule
OLCP        & 0.912 & 1.833 & 2015 & 0.890 & 4.004 & 2015 \\
OLCP-Hedge  & 0.913 & 1.824 & 2015 & 0.888 & 4.007 & 2015 \\
\bottomrule
\end{tabular}
\end{table}

The VIX-stratified table shows that high-volatility periods are substantially harder: all methods have
lower conditional coverage when VIX is high. Global adaptive methods improve high-VIX coverage
by inflating sizes, whereas OLCP and OLCP-Hedge achieve comparable or slightly better high-VIX
coverage with smaller sizes than ACI and DtACI. SPCI remains the narrowest method but undercovers
severely in both regimes. This supports the main empirical conclusion: localization improves
efficiency, online calibration helps maintain validity, and their combination gives the best overall
coverage--size tradeoff.

\end{document}